\documentclass{article} %
\usepackage[usenames,dvipsnames]{xcolor}
\usepackage{iclr2023_conference,times}

\usepackage[page]{appendix} %

\usepackage{amsmath,amsfonts,bm,bbm,gensymb}

\def\eqref#1{equation~(\ref{#1})}

\def\Appxref#1{Appx.~\ref{#1}}

\def\Algref#1{Algorithm~\ref{#1}}

\def\1{\bm{1}}

\usepackage{cancel}

\usepackage{listings}
\usepackage{enumerate}

\DeclareMathAlphabet{\mathsfit}{\encodingdefault}{\sfdefault}{m}{sl}
\SetMathAlphabet{\mathsfit}{bold}{\encodingdefault}{\sfdefault}{bx}{n}

\DeclareMathOperator*{\argmin}{arg\,min}

\usepackage{amsthm}

\makeatletter
\newtheorem*{rep@theorem}{\rep@title}
\newcommand{\newreptheorem}[2]{%
\newenvironment{rep#1}[1]{%
 \def\rep@title{#2 \ref{##1}}%
 \begin{rep@theorem}}%
 {\end{rep@theorem}}}
\makeatother

\newtheorem{theorem}{Theorem}
\newreptheorem{theorem}{Theorem}
\newtheorem{proposition}{Proposition}
\newreptheorem{proposition}{Proposition}

\newtheorem{lemma}{Lemma}
\newreptheorem{lemma}{Lemma}

\newtheorem{definition}{Definition}

\newcommand{\trace}{tr}

\usepackage{algorithm,algorithmic}

\usepackage{hyperref}
\usepackage{url}
\usepackage{subcaption}

\usepackage{multirow}

\definecolor{green}{rgb}{0.0, 0.42, 0.24} %
\definecolor{orange}{rgb}{0.8, 0.33, 0.} %
\definecolor{blue}{rgb}{0.16, 0.32, 0.75} %
\definecolor{cobalt}{rgb}{0.0, 0.28, 0.67} %
\definecolor{egred}{rgb}{1.0, 0.25, 0.25}

\definecolor{codegreen}{rgb}{0,0.6,0}
\definecolor{codegray}{rgb}{0.5,0.5,0.5}
\definecolor{codepurple}{rgb}{0.58,0,0.82}
\definecolor{backcolour}{rgb}{0.95,0.95,0.92}
\definecolor{highlight}{rgb}{0.0, 0.0, 0.0} %

\lstdefinestyle{mystyle}{
    backgroundcolor=\color{backcolour},   
    commentstyle=\color{codegreen},
    keywordstyle=\color{magenta},
    numberstyle=\tiny\color{codegray},
    stringstyle=\color{codepurple},
    basicstyle=\ttfamily\footnotesize,
    breakatwhitespace=false,         
    breaklines=true,                 
    captionpos=b,
    escapeinside={\%*}{*)},
    keepspaces=true,                 
    numbers=left,                    
    numbersep=5pt,                  
    showspaces=false,                
    showstringspaces=false,
    showtabs=false,                  
    tabsize=2
}

\usepackage{tikz}
\usetikzlibrary{matrix}
\usetikzlibrary{shapes}

\newcommand{\sgeg}{$\gamma$-EigenGame}

\newcommand{\geg}{$\Gamma$-EigenGame}
\newcommand{\gep}{SGEP}

\usepackage[normalem]{ulem}

\title{The \textcolor{highlight}{Symmetric} Generalized Eigenvalue\\Problem as a Nash Equilibrium}

\author{%
  Ian Gemp\thanks{Asterisk denotes equal contribution.} \\
  DeepMind\\
  London, UK\\
  \texttt{imgemp@deepmind.com} \\
   \And
   Charlie Chen\footnotemark[1] \\
   DeepMind\\
   London, UK\\
   \texttt{ccharlie@deepmind.com} \\
   \And
   Brian McWilliams\thanks{Work done while at DeepMind.} \\
   Google Research\\
   Z\"{u}rich, Switzerland\\
   \texttt{bmcw@google.com} \\
}

\iclrfinalcopy %
\begin{document}

\maketitle

\begin{abstract}
The symmetric generalized eigenvalue problem (\gep{}) is a fundamental concept in numerical linear algebra. It captures the solution of many classical machine learning problems such as canonical correlation analysis, independent components analysis, partial least squares, linear discriminant analysis, principal components and others. Despite this, most general solvers are prohibitively expensive when dealing with \emph{streaming data sets} (i.e., minibatches) and research has instead concentrated on finding efficient solutions to specific problem instances. In this work, we develop a game-theoretic formulation of the top-$k$ \gep{} whose Nash equilibrium is the set of generalized eigenvectors. We also present a parallelizable algorithm with guaranteed asymptotic convergence to the Nash. Current state-of-the-art methods require $\mathcal{O}(d^2k)$ runtime complexity per iteration which is prohibitively expensive when the number of dimensions ($d$) is large. We show how to modify this parallel approach to achieve $\mathcal{O}(dk)$ runtime complexity. Empirically we demonstrate that this resulting algorithm is able to solve a variety of \gep{} problem instances including a large-scale analysis of neural network activations.
\end{abstract}

\section{Introduction}\label{sec:intro}

This work considers the symmetric generalized eigenvalue problem (\gep{}),
\begin{align}
    Av &= \lambda Bv
\end{align}
where $A$ is symmetric and $B$ is symmetric, positive definite.
While the \gep{} is not a common sight in modern machine learning literature, remarkably, it underlies several fundamental problems. Most obviously, when $A=X^\top X$, $B=I$, and $X$ is a data matrix, we recover the ubiquitous SVD/PCA. However, by considering other forms of $A$ and $B$ we recover other well known problems. In general, we assume $A$ and $B$ consist of sums or expectations over outerproducts (e.g., $X^\top Y$ or $\mathbb{E}[xy^\top]$) to enable efficient matrix-vector products.
These include, but are not limited to:

\paragraph{Canonical Correlation Analysis (CCA):} Given a dataset of \emph{paired} observations (or views) $x \in \mathbb{R}^{d_x}$ and $y \in \mathbb{R}^{d_y}$ (e.g., gene expressions $x$ and medical imaging $y$ corresponding to the same patient), CCA returns the linear projections of $x$ and $y$ that are maximally correlated. CCA is particularly useful for learning multi-modal representations of data and in semi-supervised learning~\citep{mcwilliams2013correlated}; it is effectively the multi-view generalization of PCA~\citep{guo2019canonical} where $A$ and $B$ contain the cross- and auto-covariances of the two views respectively:
\nocite{tan2022canonical}
\begin{align}
    A &= \begin{bmatrix} \mathbf{0} & \mathbb{E}[x y^\top] \\ \mathbb{E}[y x^\top] & \mathbf{0} \end{bmatrix} & B &= \begin{bmatrix} \mathbb{E}[x x^\top] & \mathbf{0} \\ \mathbf{0} & \mathbb{E}[y y^\top] \end{bmatrix}.
\end{align}

\paragraph{Independent Component Analysis (ICA):}
ICA seeks the directions in the data which are most structured, or alternatively, appear least Gaussian~\citep{hyvarinen2000independent}.
A common \gep{} formulation of ICA uncovers latent variables which maximize the non-Gaussianity of the data as defined by its excess kurtosis.
ICA has famously been proposed as a solution to the so-called cocktail party source-separation problem in audio processing and has been used for denoising and more generally, the discovery of explanatory latent factors in data. Here $A$ and $B$ are the excess kurtosis and the covariance of the data respectively~\citep{parra2003blind}:
\begin{align}
    A& = \mathbb{E}[\langle x, x \rangle x x^\top] - \trace(B) B - 2 B^2 & B &= \mathbb{E}[x x^\top].
\end{align}

\paragraph{Normalized Graph Laplacians:} The graph Laplacian matrix ($L$) is central to tasks such as spectral clustering ($A=L$, $B=I$) where its eigenvectors are known to solve a relaxation of min-cut~\citep{von2007tutorial}. Alternatives, such as the random walk normalized Laplacian ($A=L$, $B$ is the diagonal node-degree matrix), approximate other min-cut objectives. These normalized variants, in particular, are important to computing representations for learning value functions in reinforcement learning such as successor features~\citep{machado2017laplacian,stachenfeld2014design,machado2017eigenoption}, an extension of proto-value functions~\citep{mahadevan2005proto} which uses the un-normalized graph Laplacian ($A=L$, $B=I$).

Partial least squares (PLS) can be formualted similarly to CCA and finds extensive use in chemometrics~\citep{boucher2015study}, medical domains~\citep{altmannpartial} and beyond~\citep{mcwilliams2010sparse}. Likewise, linear discriminant analysis (LDA) can be formulated as a \gep{} and learns a label-aware projection of the data that separates classes well~\citep{rao1948utilization}. More examples and uses of the \gep{} can be found in~\citep{bie2005eigenproblems,borga1997unified}. We now shift focus to the mathematical properties and challenges of the corresponding \gep{}.

In this work, we assume the matrices $A$ and $B$ above can either be defined using expectations under a data distribution (e.g., $\mathbb{E}_{x \sim p(x)} [ x x^\top ]$) or means over a finite sample dataset (e.g., $\frac{1}{n} X^\top X$ where $X \in \mathbb{R}^{n \times d_x}$). In either case, we typically assume the data has mean zero unless specified otherwise.

Note that the \gep{}, $Av = \lambda Bv$, is similar to the eigenvalue problem $B^{-1}A v = \lambda v$.
There are two reasons for working with the \gep{} instead: 1) inverting $B$ is prohibitively expensive for a large matrix and 2) while $A$ and $B \succ 0$ are symmetric, $B^{-1}A$ is not, which hides useful information about the eigenvalues and eigenvectors (they are necessarily real and $B$-orthogonal). This also highlights that the \gep{} is a fundamentally more challenging problem than SVD and why a direct application of previous game-theoretic approaches such as~\citep{gemp2021eigengame,gemp2022eigengame} is not possible.

The complexity of solving the \gep{} is $\mathcal{O}(d^3)$ where $d$ is the dimension of the square matrix $A$ (equiv. $B$). Several libraries exist for solving the \gep{} in-memory~\citep{tzounas2020comparison}.
There is also a vast numerics literature we cannot do justice that considers large matrices~\citep{sorensen2002numerical}.

We specifically focus on the stochastic, streaming data setting which is of particular interest to machine learning methods which learn by iterating over small minibatches of data (e.g., stochastic gradient descent). Under this setting, machine learning research has developed simple approximate solvers for singular value decomposition (SVD) that scale to very large datasets~\citep{allen2017first}. Similarly, in this work, we contribute a simple, elegant solution to the \gep{}, including
\begin{itemize}
    \item A game whose Nash equilibrium is the top-$k$ \gep{} solution,
    \item An easily parallelizable algorithm with $\mathcal{O}(dk)$ per-iteration complexity relying only on matrix-vector products,
    \item An empirical analysis of neural similarity on activations $1000\times$ larger than prior work.
\end{itemize}
The game and accompanying algorithm are developed synergistically to achieve a formulation that is amenable to analysis and naturally leads to an elegant and efficient algorithm.

\section{Generalized EigenGame: Players, Strategies, and Utilities}\label{sec:game}

In this work, we take the approach of defining the top-$k$ \gep{} as a $k$-player game. It is an open question how to define a $k$-player game appropriately such that key properties of the \gep{} are captured\footnote{The Courant-Fischer min-max principle poses the $i$th generalized eigenvalue as the solution to a two-player, zero-sum game~\citep{parlett1998symmetric}\textemdash see \Appxref{cf_minmax} for further discussion.}.
As argued in previous work~\citep{gemp2021eigengame,gemp2022eigengame}, game formulations make obvious how computation can be distributed over players, leading to high parallelization, which is critical for processing large datasets. They have also clarified geometrical properties of the problem.

Specifically, we are interested in solving the top-$k$ \gep{} which means we are interested in finding the (unit-norm) generalized eigenvectors $v_i$ associated with the top-$k$ largest generalized eigenvalues $\lambda_i$. Therefore, let there be $k$ \textbf{players} denoted $i \in \{1, \ldots, k\}$, and let each select a vector $\hat{v}_i$ (\textbf{strategy}) from the unit-sphere $\mathcal{S}^{d-1}$ (\emph{strategy space}). We define player $i$'s \textbf{utility} function conditioned on its parents (players $j < i$) as follows:
\begin{align}
    u_i(\hat{v}_i \vert \hat{v}_{j<i}) &= \overbrace{\frac{\langle \hat{v}_i, A \hat{v}_i \rangle}{\langle \hat{v}_i, B \hat{v}_i\rangle}}^{\stackrel{\text{\emph{generalized}}}{\text{Rayleigh Quotient}}} - \sum_{j < i} \frac{\langle \hat{v}_j, A \hat{v}_j \rangle \langle \hat{v}_i, B \hat{v}_j \rangle^2}{\langle \hat{v}_j, B \hat{v}_j \rangle^2 \langle \hat{v}_i, B \hat{v}_i \rangle} %
    \\ &= 
    \underbrace{\hat{\lambda}_i}_{\text{\textcolor{blue}{reward}}} - \sum_{j < i} \underbrace{\hat{\lambda}_j \langle \textcolor{green}{\hat{y}_i}, B \textcolor{green}{\hat{y}_j} \rangle^2}_{\text{\textcolor{red}{penalty}}} \quad \text{ where $\textcolor{green}{\hat{y}_i} = \frac{\hat{v}_i}{||\hat{v}_i||_B}$,} \label{eqn:util_new}
\end{align}
$\hat{\lambda}_i = \frac{\langle \hat{v}_i, A \hat{v}_i \rangle}{\langle \hat{v}_i, B \hat{v}_i \rangle}$, and $||z||_B = \sqrt{\langle z, B z \rangle}$.

Player $i$'s utility has an intuitive explanation. The first term is recognized as the generalized Rayleigh quotient which can be derived by left multiplying both sides of the \gep{} ($v^\top A v = \lambda v^\top B v$) and solving for $\lambda$. Note that the generalized eigenvectors are guaranteed to be $B$-orthogonal, i.e., $v_i^\top B v_j = 0$ for all $i \ne j$ (\Appxref{app:prelims} Lemma~\ref{b_orth}). Therefore, the \emph{\textcolor{blue}{reward}} term incentivizes players to find directions that result in large eigenvalues, but are simultaneously \emph{\textcolor{red}{penalized}} for choosing directions that align with their parents (players with index less than $i$, higher in the hierarchy). Finally, the penalty coefficient $\hat{\lambda}_j$ serves to balance the magnitude of the penalty terms with the reward term such that players have no incentive to ``overlap'' with parents. \textcolor{highlight}{In \Appxref{app:well_posed} Proposition~\ref{prop:deflation}, we derive these same utilities via a \emph{deflation} perspective.} Next, we formally prove these utilities are well-posed in the sense that, given exact parents, their optima coincide with the top-$k$ \gep{} solution.

\begin{lemma}[Well-posed Utilities]
\label{well_posed_utils}
Given exact parents and assuming the top-$k$ eigenvalues of $B^{-1}A$ are distinct and positive, the maximizer of player $i$'s utility is the unique generalized eigenvector $v_i$ (up to sign, i.e., $-v_i$ is also valid).
\end{lemma}

Note that $\hat{\lambda}_i = \frac{\langle \hat{v}_i, A \hat{v}_i \rangle}{\langle \hat{v}_i, B \hat{v}_i \rangle} = \frac{\langle \hat{v}_i / ||\hat{v}_i||_B, A \hat{v}_i / ||\hat{v}_i||_B \rangle}{\langle \hat{v}_i / ||\hat{v}_i||_B, B \hat{v}_i / ||\hat{v}_i||_B \rangle} = \frac{\langle \hat{y}_i, A \hat{y}_i \rangle}{\langle \hat{y}_i, B \hat{y}_i \rangle}$, therefore, the results above still hold for utilities defined using vectors constrained to the unit ellipsoid, $||\hat{y}_i||_B = 1$, rather than the unit-sphere, $||\hat{v}_i||_{I} = 1$.
However, in our setting, $B$ is a massive matrix which can never be explicitly constructed and instead only observed via minibatches. It is then not clear how to handle the constraint $||\hat{y}_i||_B = 1$. We therefore only consider an approach that assumes $||\hat{v}_i||_{I} = 1$. %

Next, we provide intuition for the shape of these utilities. Surprisingly, while non-concave, we prove analytically in \Appxref{app:well_posed} that they have a simple sinusoidal shape. A numerical illustration is given in Figure~\ref{fig:warped_sinusoid} to help the reader visualize this property.

\begin{proposition}[Utility Shape]
\label{util_shape}
Each player's utility is periodic in the angular deviation ($\theta$) along the sphere. Its shape is sinusoidal, but with its angular axis ($\theta$) smoothly deformed as a function of $B$. Most importantly, every local maximum is a global maximum.
\end{proposition}

Figure~\ref{fig:warped_sinusoid} illustrates a primary difficulty of solving the \gep{} over SVD. Due to the extreme differences in curvature caused by the $B$ matrix, the \gep{} should benefit from optimizers employing adaptive per-dimension learning rates. To our knowledge, this 1-d visualization of the difficulty of the \gep{} is novel and we exploit this insight in experiments.

\begin{figure}[ht!]
    \centering
    \includegraphics[width=0.75\textwidth]{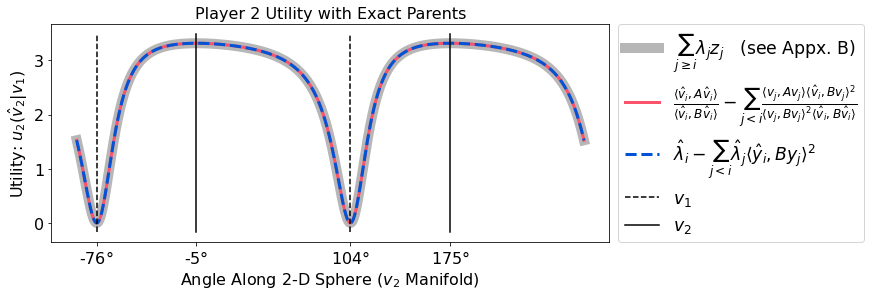}
    \caption{Each player's utility is a sinusoid on the sphere warped tangentially along the axis of angular deviation according to $B$; values for $A$ and $B$ used in this example are given in \Appxref{app:well_posed}. Three mathematical representations of the utility are plotted; their equivalence is supported by the overlapping curves. If player $2$ aligns with the top eigenvector (dashed vertical), they receive zero utility. If they align with the second eigenvector (solid vertical), they receive $\lambda_2$ (optimal) as reward. If $B=I$ as in SVD/PCA, the vertical lines indicating the minima and maxima would be separated by exactly $90\degree$. In this case, the matrix $B$ redefines what it means for two vectors to be orthogonal ($\langle \hat{v}_i, B \hat{v}_j \rangle = 0$), so that the vectors are $71\degree$ (equivalently, $180\degree-71\degree=109\degree$) from each other.}
    \label{fig:warped_sinusoid}
    \vspace{-0.4cm}
\end{figure}

Finally, we formally define our proposed game and prove its equilibrium constitutes the top-$k$ \gep{} solution. We use the Greek letter \emph{\textbf{g}amma} to denote \emph{\textbf{g}eneralized}, and we differentiate between the game and the algorithm with upper $\Gamma$ and lower case $\gamma$ respectively.

\begin{definition}[\geg{}]
\label{def:geg}
Let \geg{} be the game with players $i \in \{1, \ldots, k\}$, their strategy spaces $\hat{v}_i \in \mathcal{S}^{d-1}$, and their utilities $u_i$ as defined in~\eqref{eqn:util_new}.
\end{definition}

\begin{theorem}[Nash Property]
Assuming the top-$k$ generalized eigenvalues of the generalized eigenvalue problem $Av = \lambda Bv$ are positive and distinct,  their corresponding generalized eigenvectors form the unique, strict Nash equilibrium of \geg{}.
\end{theorem}
\begin{proof}
Lemma~\ref{well_posed_utils} proves that each generalized eigenvector $v_i$ ($i \in \{1, \ldots, k\}$) is the unique best response to $v_{-i}$, which implies the entire set constitutes the unique Nash equilibrium.
\end{proof}
\section{Algorithm: Unbiased Player Updates and Auxiliary Variables}\label{sec:update}

Given that \geg{} suitably captures the top-$k$ \gep{}, we now develop an iterative algorithm to approximate its solution. The basic approach we take is to perform parallel gradient ascent on all player utilities simultaneously. We focus on this approach in particular because it aligns with the predominant machine learning paradigm and hardware. We will first write down the gradient of each player's utility and then introduce several simplifications for the purpose of enabling unbiased estimates in the stochastic setting.

Up to scaling factors, the gradient of player $i$'s utility function with respect to $\hat{v}_i$ is
\begin{align}
    \frac{(\hat{v}_i^\top B \hat{v}_i) A \hat{v}_i - (\hat{v}_i^\top A \hat{v}_i) B \hat{v}_i}{\langle \hat{v}_i, B \hat{v}_i \rangle^2} - \sum_{j < i} \frac{\hat{\lambda}_j}{\langle \hat{v}_j, B \hat{v}_j \rangle} (\hat{v}_i^\top B \hat{v}_j) \frac{\big[ \langle \hat{v}_i, B \hat{v}_i \rangle B \hat{v}_j - \langle \hat{v}_i, B \hat{v}_j \rangle B \hat{v}_i \big]}{\langle \hat{v}_i, B \hat{v}_i \rangle^2}. \label{eq:util_grad}
\end{align}

See Lemma~\ref{lemma:grad_derivation} in \Appxref{app:well_posed} for a derivation of the gradient. Recall that $B$ is a matrix that we intend to estimate with samples, i.e., it is a random variable, and it appears several times in the denominator of the gradient. Obtaining unbiased estimates of inverses of random variables is difficult (e.g., the naive approach gives an overestimate; $\mathbb{E}[1/x] \ge 1/\mathbb{E}[x]$ by Jensen's inequality).
We can remove the scalar $\langle \hat{v}_i, B \hat{v}_i \rangle^2$ in the denominator because it is common to all terms and will not change the direction of the gradient nor the location of fixed points; this step is critical to the design of our stochastic algorithm which we will explain later. We also use the following two additional relations:
\begin{enumerate}[(i)]
    \item $\hat{\lambda}_j \langle \hat{v}_i, B \hat{v}_j \rangle = \langle \hat{v}_i, A \hat{v}_j \rangle$ if player $i$'s parents match their true solutions, i.e., $\hat{v}_{j<i} = v_{j<i}$, \label{rel_1}
    \item $\sqrt{\langle \hat{v}_j, B \hat{v}_j \rangle} = ||\hat{v}_j||_B$ is strictly positive and real-valued because $B \succ 0$, \label{rel_2}
\end{enumerate}
to arrive at the simplified update direction
\begin{align}
    \tilde{\nabla}_i &= \overbrace{(\hat{v}_i^\top B \hat{v}_i) A \hat{v}_i - (\hat{v}_i^\top A \hat{v}_i) B \hat{v}_i}^{\text{reward}} - \sum_{j < i} \overbrace{(\hat{v}_i^\top A \textcolor{green}{\hat{y}_j}) \big[ \langle \hat{v}_i, B \hat{v}_i \rangle B \textcolor{green}{\hat{y}_j} - \langle \hat{v}_i, B \textcolor{green}{\hat{y}_j} \rangle B \hat{v}_i}^{\text{penalty}} \big]. \label{eq:update}
\end{align}
Simplifying the gradient using~(\ref{rel_1}) is sound because the hierarchy of players ensures the parents will be learned exactly asymptotically. For instance, player $1$'s update has no penalty terms and so will converge asymptotically. The argument then proceeds by induction.

Note that $B$ still appears in the denominator via the $\textcolor{green}{\hat{y}_j}$ terms (recall \eqref{eqn:util_new}). We will revisit this issue later, but for now we will show this update converges to the desired solution given exact estimates of expectations (full-batch setting). Lemma~\ref{steepest_ascent} is a stepping stone to proving convergence with arbitrary parents in Theorem~\ref{thm:global_conv}.

\begin{lemma}\label{steepest_ascent}
The direction $\tilde{\nabla}_i$ defined in~\eqref{eq:update} is a steepest ascent direction on utility $u_i(\hat{v}_i \vert \hat{v}_{j < i})$ given exact parents $\hat{v}_{j < i} = v_{j < i}$.
\end{lemma}
\begin{proof}
This fact follows from the above argument that removing a positive scalar multiplier does not change the direction of the gradient of $u_i$ w.r.t. $\hat{v}_i$ and applying relation~(\ref{rel_1}).
\end{proof}

We present the deterministic version of \sgeg{} in \Algref{alg:det_geg} where $k$ players use $\tilde{\nabla}_i$ in~(\ref{eq:update}) to maximize their utilities in parallel (see \textbf{parfor}-loop below). While simultaneous gradient ascent fails to converge to Nash equilibria in games in general, it succeeds in this case because the hierarchy we impose ensures each player has a unique best response (Lemma~\ref{well_posed_utils}); this type of procedure is known as \emph{iterative strict dominance} in the game theory literature. Theorem~\ref{thm:global_conv}, proven in~\Appxref{app:conv}, guarantees it converges asymptotically to the true solution.

\begin{algorithm}[t]
\caption{Deterministic / Full-batch \sgeg{}}
\label{alg:det_geg}
\begin{algorithmic}[1]
    \STATE Given: $A \in \mathbb{R}^{d \times d}$ and $B \in \mathbb{R}^{d \times d}$, step size sequence $\eta_t$, and number of iterations $T$.
    \STATE $\hat{v}_i \sim \mathcal{S}^{d-1}$, i.e., $\hat{v}_i \sim \mathcal{N}(\mathbf{0}_d, \mathbf{I}_d); \hat{v}_i \leftarrow \hat{v}_i / ||\hat{v}_i|$ for all $i$
    \FOR{$t = 1: T$}
        \PARFOR{$i = 1: k$}
            \STATE $\hat{y}_j = \frac{\hat{v}_j}{\sqrt{\langle \hat{v}_j, B\hat{v}_j \rangle }}$
            \STATE $\texttt{rewards} \leftarrow (\hat{v}_i^\top B \hat{v}_i) A \hat{v}_i - (\hat{v}_i^\top A \hat{v}_i) B \hat{v}_i$
            \STATE $\texttt{penalties} \leftarrow \sum_{j < i} (\hat{v}_i^\top A \hat{y}_j) \big[ \langle \hat{v}_i, B \hat{v}_i \rangle B \hat{y}_j - \langle \hat{v}_i, B \hat{y}_j \rangle B \hat{v}_i \big]$
            \STATE $\tilde{\nabla}_{i} \leftarrow \texttt{rewards} - \texttt{penalties}$
        \STATE $\hat{v}_i' \leftarrow \hat{v}_i + \eta_t \tilde{\nabla}_{i}$
        \STATE $\hat{v}_i \leftarrow \frac{\hat{v}_i'}{|| \hat{v}_i' ||}$
        \ENDPARFOR
    \ENDFOR
    \STATE return all $\hat{v}_i$
\end{algorithmic}
\end{algorithm}

\begin{theorem}[Deterministic / Full-batch Global Convergence]\label{thm:global_conv}
Given a symmetric matrix $A$ and symmetric positive definite matrix $B$ where the top-$k$ eigengaps of $B^{-1}A$ are positive along with a square-summable, not summable step size sequence $\eta_t$ (e.g., $1/t$), \Algref{alg:det_geg} converges to the top-$k$ generalized eigenvectors asymptotically ($\lim_{T \rightarrow \infty}$) with probability $1$.
\end{theorem}

\begin{algorithm}[!t]
\caption{Stochastic \sgeg{}}
\label{alg:gen_eg}
\begin{algorithmic}[1]
    \STATE Given: paired data streams $X_t \in \mathbb{R}^{b \times d_x}$ and $Y_t \in \mathbb{R}^{b \times d_y}$, number of parallel machines $M$ per player (minibatch size per machine $b'=\frac{b}{M}$), step size sequences $\eta_t$ and $\gamma_t$, scalar $\rho$ lower bounding $\sigma_{min}(B)$, and number of iterations $T$.
    \STATE $\hat{v}_i \sim \mathcal{S}^{d-1}$, i.e., $\hat{v}_i \sim \mathcal{N}(\mathbf{0}_d, \mathbf{I}_d); \hat{v}_i \leftarrow \hat{v}_i / ||\hat{v}_i|$ for all $i$
    \STATE $[B\hat{v}]_i \leftarrow \hat{v}_i^0$ for all $i$
    \FOR{$t = 1: T$}
        \PARFOR{$i = 1: k$}
            \PARFOR{$m = 1: M$}
            \STATE Construct $A_{tm}$ and $B_{tm}$ (*unbiased estimates using independent data batches)
            \STATE $\hat{y}_j = \frac{\hat{v}_j}{\sqrt{\textcolor{red}{\max(} \langle \hat{v}_j, \textcolor{blue}{[B\hat{v}]_j} \rangle \textcolor{red}{, \rho)}}}$ \label{eq:clipping}
            \STATE $\textcolor{blue}{[B\hat{y}]_j} = \frac{\textcolor{blue}{[B\hat{v}]_j}}{\sqrt{\textcolor{red}{\max(} \langle \hat{v}_j, \textcolor{blue}{[B\hat{v}]_j} \rangle \textcolor{red}{, \rho)}}}$  \label{eq:Byj}
            \STATE $\texttt{rewards} \leftarrow (\hat{v}_i^\top B_{tm} \hat{v}_i) A_{tm} \hat{v}_i - (\hat{v}_i^\top A_{tm} \hat{v}_i) B_{tm} \hat{v}_i$
            \STATE $\texttt{penalties} \leftarrow \sum_{j < i} (\hat{v}_i^\top A_{tm} \hat{y}_j) \big[ \langle \hat{v}_i, B_{tm} \hat{v}_i \rangle \textcolor{blue}{[B \hat{y}]_j} - \langle \hat{v}_i, \textcolor{blue}{[B \hat{y}]_j} \rangle B_{tm} \hat{v}_i \big]$
            \STATE $\tilde{\nabla}_{im} \leftarrow \texttt{rewards} - \texttt{penalties}$ \label{eq:gevp_update}
            \STATE \textcolor{blue}{$\nabla^{Bv}_{im} = (B_{tm}\hat{v}_i - [B\hat{v}]_i)$}
            \ENDPARFOR
        \STATE $\tilde{\nabla}_{i} \leftarrow \frac{1}{M} \sum_{m} [ \tilde{\nabla}_{im} ]$
        \STATE $\hat{v}_i' \leftarrow \hat{v}_i + \eta_t \tilde{\nabla}_{i}$
        \STATE $\hat{v}_i \leftarrow \frac{\hat{v}_i'}{|| \hat{v}_i' ||}$
        \STATE \textcolor{blue}{$\nabla^{Bv}_{i} \leftarrow \frac{1}{M} \sum_m [\nabla^{Bv}_{im}]$}
        \STATE \textcolor{blue}{$[B\hat{v}]_i \leftarrow [B\hat{v}]_i + \gamma_t \nabla^{Bv}_{i}$}
        \ENDPARFOR
    \ENDFOR
    \STATE return all $\hat{v}_i$
\end{algorithmic}
\end{algorithm}

In the big data setting, $A$ and $B$ are statistical estimates, i.e., expectations of quantities over large datasets. Precomputing exact estimates is computationally expensive, so we assume a data model that allows drawing small \emph{minibatches} of data at a time. Under such a model, stochastic approximation theory typically guarantees that as long as the update directions are \emph{unbiased}, i.e., equal in expectation to the updates with exact estimates, then an appropriate algorithm will converge to the true solution.

In order to construct an unbiased update direction given access to minibatches of data, we need to draw multiple minibatches independently at random. We can construct an unbiased estimate of products of expectations, e.g., $(\hat{v}_i^\top B \hat{v}_i) A \hat{v}_i$, by drawing an independent batch for each, e.g., one for $B$ and one for $A$. However, the $B$ that appears in the denominator of $\hat{y}_j$ is problematic; we cannot construct an unbiased estimate of the inverse of a random variable.

These problematic $\hat{y}_j$ terms only appear in the penalties, which are a function of the parents' eigenvector approximations. The first eigenvector has no parents, and so we can easily construct an unbiased estimate for it using multiple minibatches. We can then construct an unbiased estimate for each subsequent player by inductive reasoning. Intuitively, once the parents have been learned, $\hat{v}_j$ should be stable and so it should be possible to estimate $B\hat{v}_j$ from a running average, and in turn, $\hat{y}_j$. This suggests introducing an auxiliary variable, denoted $\textcolor{blue}{[B\hat{v}]_j}$ to track the running averages of $B\hat{v}_j$ (a similar approach is employed in~\citep{pfau2018spectral}). This effectively replaces $B\hat{y}_j$ with a non-random variable, avoiding the bias dilemma, at the expense of doubling the number of variables. Note that introducing this auxiliary variable implies the inner product $\langle \hat{v}_j, [B\hat{v}]_j \rangle$ may not be positive definite, therefore, we manually clip the result to be greater than or equal to $\rho$, the minimum singular value of $B$.

Precise pseudocode is given in \Algref{alg:gen_eg}. Differences to \Algref{alg:det_geg} are highlighted in color (auxiliary differences in \textcolor{blue}{blue}, clipping in \textcolor{red}{red}). We point out that introducing an auxiliary variable for player $i$ to track $[B\hat{v}]_i$ is not feasible because unlike player $i$'s parents' variables, $\hat{v}_i$ cannot be assumed non-stationary. This is why removing $\langle \hat{v}_i, B \hat{v}_i \rangle^2$ earlier from the denominator of~\eqref{eq:util_grad} was critical. Lastly, note that these modifications to the update are derived using an understanding of the intended computation and theoretical considerations; put shortly, \emph{autograd} libraries will not uncover this solution.
See~\Appxref{app:conv_sto} for more discussion and analysis of~\Algref{alg:gen_eg}.

\begin{figure}[ht!]
    \centering
    \includegraphics[width=0.95\textwidth]{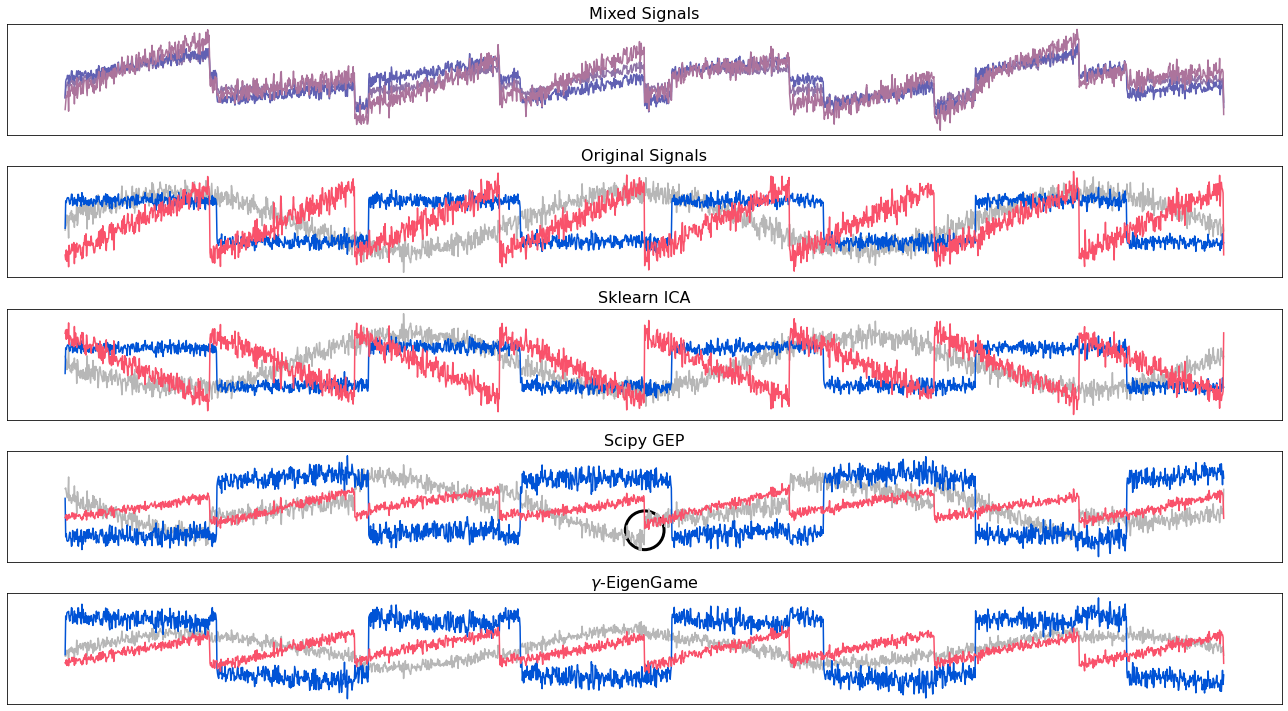}
    \caption[Blind Source Separation. \Algref{alg:gen_eg} (\sgeg{}) run for $1000$ epochs with minibatches of size $\frac{n}{4}$ subsampled i.i.d. from the dataset recovers the three original signals from the linearly mixed signals. Scikit-learn's FastICA also recovers the signals by maximizing an alternative measure of non-Gaussianity. Directly solving the \gep{} using \texttt{scipy.linalg.eigh(A, B)} fails to cleanly recover the gray sinusoid (black circle highlights a discontinuity) due to overfitting to the sample dataset. We confirm this by training \sgeg{} for many more iterations and show it exhibits similar artifacts in~\Appxref{app:more_experiments}.]
    {Blind Source Separation. \Algref{alg:gen_eg} (\sgeg{}) run for $1000$ epochs with minibatches of size $\frac{n}{4}$ subsampled i.i.d. from the dataset recovers the three original signals from the linearly mixed signals. Scikit-learn's FastICA\footnotemark also recovers the signals by maximizing an alternative measure of non-Gaussianity. Directly solving the \gep{} using \texttt{scipy.linalg.eigh(A, B)} fails to cleanly recover the gray sinusoid (black circle highlights a discontinuity) due to overfitting to the sample dataset. We confirm this by training \sgeg{} for many more iterations and show it exhibits similar artifacts  in~\Appxref{app:more_experiments}.}
    \label{fig:ica}
    \vspace{-0.8cm}
\end{figure}

\textbf{Computational Complexity and Parallelization.} The naive, per-iteration runtime and work costs of this update are $\mathcal{O}(bdk^2)$ with batch size $b$, but \textcolor{highlight}{due to the simplicty of the update (i.e., purely matrix-vector products)}, there are several opportunities for both model and data parallelism to reduce runtime cost to $\mathcal{O}(dk)$ (see \Appxref{app:complexity} for steps). \textcolor{highlight}{Note that low aggregate batch sizes can induce high gradient variance, slowing convergence~\citep{durmus2021riemannian}, making data parallelism desireable.}

To give a concrete example, if each player (model) parallelizes over $M=b$ machines (data) as indicated by the two \textbf{parfor}-loops, the complexity reduces to $\mathcal{O}(dk)$. This is easy to implement with modern libraries, e.g., \texttt{pmap} using Jax, and as in prior work~\citep{gemp2021eigengame}, the communication of parents $v_{j<i}$ between machines is efficient in systems with fast interconnects (e.g., TPUs) although this presents a bottleneck we hope to alleviate in future work. Alternative parallel implementations are discussed in~\Appxref{app:parallel}. Lastly, the update consists purely of inexpensive elementwise operations and matrix-vector products that can be computed quickly on deep learning hardware (e.g., GPUs and TPUs); unlike previous state-of-the-art in~\citep{meng2021online}, no calls to CPU-bound linear algebra subroutines are necessary.
\section{Related Work}\label{sec:related_work}

\footnotetext{Run with \texttt{logcosh} approximation to negentropy (see~\citep{hyvarinen2000independent} for explanation).}

The \gep{} is a fundamental problem in numerical linear algebra with numerous applications in machine learning and statistics. 
There is a long history in numerical computing of solving large \gep{} problems~\textcolor{highlight}{\citep{sorensen2002numerical,knyazev1994preconditioned,golub2002inverse,aliaga2012solving,klinvex2013parallel}}. Many methods iterate with what can be viewed as ``gradient-like'' updates~\citep{d1982group,d1992iterative}, however, they are not immediately applicable in the stochastic, streaming data setting.
To our knowledge, efficient approaches for the \gep{} or specific sub-problems (e.g., CCA) scale at best $\mathcal{O}(d^2 k)$ in the streaming data setting.

\citet{ge2016efficient} give an algorithm for top-$k$ \gep{} that makes repeated use of a linear system solver to approximate the subspace of the true generalized eigenvectors, but may return an arbitrary rotation of the solution. While their method is theoretically efficient, it requires precomputing $A$ and $B$ which prohibits its use in a streaming data setting.
The sequential least squares CCA algorithm proposed by~\citet{wang2016efficient} similarly requires access to the full dataset up front, however, in their case, it is to ensure the generalized eigenvectors are exactly unit-norm relative to the matrix $B$.
\citet{allen2017doubly} develop a \gep{} algorithm that is theoretically linear in the size of the input ($nd$) and $k$, however, they assume access to the entire dataset (non-streaming).
\textcolor{highlight}{LOBPCG~\citep{knyazev2017recent} is a non-streaming, Rayleigh maximizing technique with line-search and a preconditioned gradient.}

\citet{arora2017stochastic} propose a convex relaxation of the CCA problem along with a streaming algorithm with convergence guarantees. However, instead of learning $V_x \in \mathbb{R}^{d_x}$ and $V_y \in \mathbb{R}^{d_y}$ directly, it learns $M = V_x V_y^\top \in \mathbb{R}^{d_x \times d_y}$ which is prohibitively expensive to store in memory for high-dimensional problems. Moreover, the complexity of this algorithm is $\mathcal{O}(d^3)$ due to an expensive projection step requiring an SVD of $M$. They propose an alternative version \emph{without} guarantees that reduces the cost per iteration to $\mathcal{O}(dk^2)$.

\begin{figure}[ht!]
    \begin{subfigure}[b]{\textwidth}
    \centering
    \includegraphics[width=0.325\textwidth]{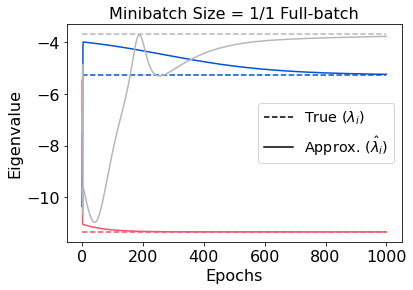}
    \includegraphics[width=0.325\textwidth]{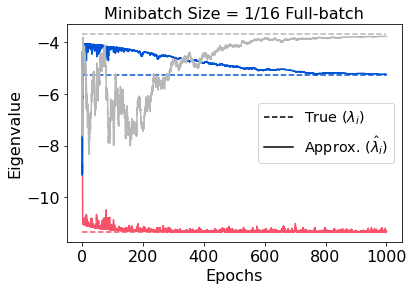}
    \includegraphics[width=0.325\textwidth]{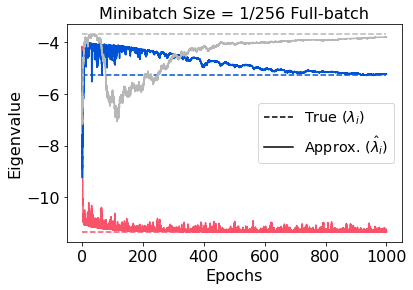}
    \end{subfigure}
    \begin{subfigure}[b]{\textwidth}
    \includegraphics[width=0.325\textwidth]{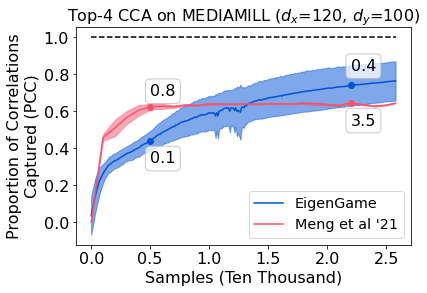}
    \includegraphics[width=0.325\textwidth]{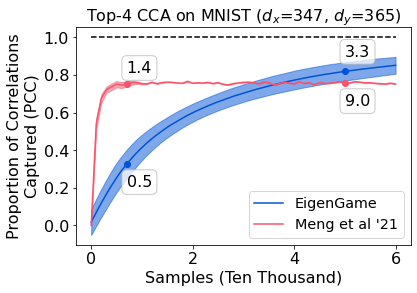}
    \includegraphics[width=0.325\textwidth]{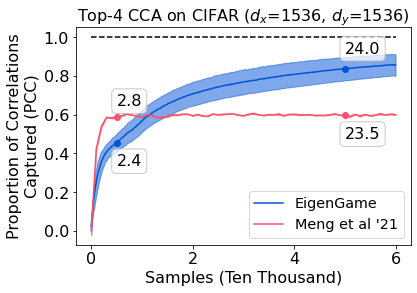}
    \end{subfigure}
    \caption{\textbf{Top:} \sgeg{} converges to the true \gep{} solution regardless of minibatch size in support of the unbiased nature of the derived update scheme (\Algref{alg:gen_eg}). \textcolor{highlight}{See~\Appxref{app:more_experiments:seq_vs_par} for additional experiments showing learning in sequence performs worse than its parallel counterpart above.} \textbf{Bottom:}
\sgeg{} compared to~\citep{meng2021online} on \emph{proportion of correlations captured}\textemdash $\sum_i^k \hat{\lambda}_i / \sum_i^k \lambda_i$. Shading indicates $\pm 1$ stdev. Markers indicate runtime in seconds.}
    \label{fig:unbiased_update}
    \vspace{-0.3cm}
\end{figure}

\citet{gao2019stochastic} also consider the streaming setting, but instead focuse on top-$1$ CCA and like~\citep{wang2016efficient} and~\citep{allen2017doubly}, use shift-invert preconditioning to accelerate convergence.
\citet{bhatia2018gen} solves top-$1$ \gep{} in the streaming setting.%

Most recently, \citet{meng2021online} proposed a method to estimate top-$k$ CCA in a streaming setting. Their algorithm requires several expensive Riemmanian optimization subroutines, giving a per iteration complexity of $\mathcal{O}(d^2k)$. Their convergence guarantee is in terms of subspace error, so as mentioned above, the projection matrices $V_x$ and $V_y$ may be rotations of their ordered (by correlation) counterparts. Their approach is the current state-of-the-art when considering CCA in the streaming setting for large datasets.

\section{Experiments}\label{sec:experiments}

\nocite{krizhevsky2009learning,deng2012mnist,snoek2006challenge}

We demonstrate our proposed stochastic approach, \Algref{alg:gen_eg}, on solving ICA and CCA via their \gep{} formulations, and provide empirical support for its veracity. A Jax implementation is available at {\small \url{github.com/deepmind/eigengame}}. Scipy's \texttt{linalg.eigh(A, B)}\citep{scipy} is treated as ground truth when the data size permits. Hyperparameters are listed in~\Appxref{app:hyps}.

\textbf{Independent Components Analysis.}
ICA can be used to disentangle mixed signals such as in the cocktail party problem. Here, we use the \gep{} formulation to unmix three linearly mixed signals. Note that because the \gep{} learns a linear unmixing of the data, the magnitude (and sign) of the original signals cannot be learned. Any change in the magnitude of a signal extracted by the \gep{} can be offset by adjusting the magnitude and sign of a mixing weight.

We replicate a synthetic experiment from \texttt{scikit-learn}\citep{scikit-learn} and compare \Algref{alg:gen_eg} to several approaches. Figure~\ref{fig:ica} shows our stochastic approach (\sgeg{}) is able to recover the shapes of the original signals (length $n=2000$ time series).

\textbf{Implicit Regularization via Fixed Step Size Updates.} Note that if we run \Algref{alg:gen_eg} for $100\times$ more iterations with $1/10$th the step size, we converge to the exact \gep{} solution (as found by \texttt{scipy}) and see similar artifacts in the extracted signals due to overfitting. Recently, \citet{durmus2021riemannian} proved that fixed step size Riemannian approximation schemes converge to a stationary distribution around their solutions, which suggests \sgeg{} enjoys a natural regularization property and explains its high performance on this the unmixing task. In~\Appxref{app:more_experiments}, we show that it is difficult to achieve similar results with \texttt{scipy} by regularizing $A$ or $B$ directly (e.g., $A + \epsilon I$) prior to calling \texttt{scipy.linalg.eigh}.

\begin{figure}[t]
    \centering
    \includegraphics[width=0.325\textwidth]{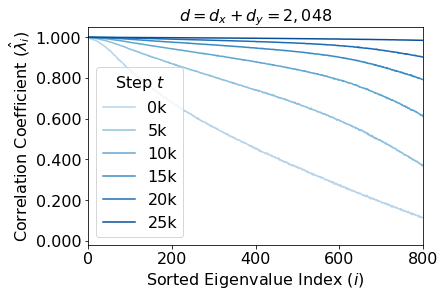} %
    \includegraphics[width=0.325\textwidth]{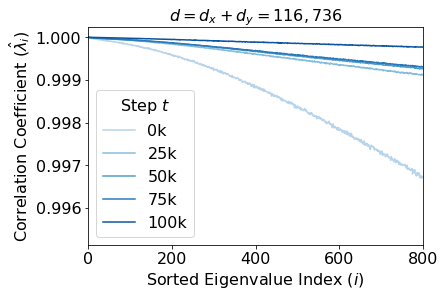} %
    \includegraphics[width=0.325\textwidth]{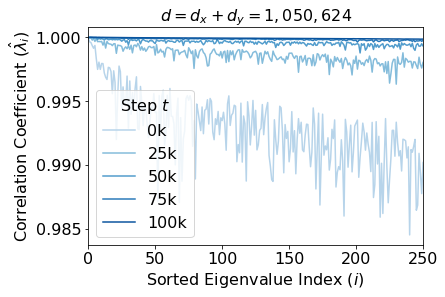} %
    \caption{\sgeg{} compares the representation (activations) of a deep network trained on CIFAR-10 for $t$ steps to that of its final learned representation ($25k$ or $100k$ steps). Curves with higher correlation coefficients indicate more similar representations.}
    \label{fig:cifar10_analysis}
    \vspace{-0.4cm}
\end{figure}

\textbf{Unbiased Updates.}
We empirically support our claim that the fixed point of \Algref{alg:gen_eg} is unbiased regardless of minibatch size.\footnote{The gray line converges last because we chose to minimize rather than maximize kurtosis.} Not only does \sgeg{} recover the same generalized eigenvalues, but the top row of figure~\ref{fig:unbiased_update} also suggests that the algorithm takes a similar trajectory for each minibatch size.

\paragraph{Canonical Correlations Analysis.}\label{sec:experiments:cca:comparison}
Here, we use \sgeg{} to linearly project multimodal datasets into lower-dimensional spaces such that they are maximally correlated.
As discussed in related work, several approaches have been developed to extend CCA to streaming, high-dimensional datasets. Recall that our approach has per-iteration complexity $\mathcal{O}(bdk)$ with the previous state-of-the-art in the streaming setting having $\mathcal{O}(d^2k)$~\citep{meng2021online}. We replicate the experiments of ~\citep{meng2021online} and compare against their approach on three datasets.

The bottom row of figure~\ref{fig:unbiased_update} shows our approach is competitive with~\citep{meng2021online}.
We also point out that while the previous approach by~\citet{meng2021online} enjoys theoretical convergence guarantees with rates, it appears to slow in progress near a biased solution.
These datasets are low dimensional ($d \le 3072$), so we are able to obtain ground truth eigenvectors efficiently using \texttt{scipy}. Our next set of experiments considers much higher dimensional where a naive call to \texttt{scipy} fails.

\nocite{matlab2018,numpy}

\textbf{Large-scale Neural Network Analysis.}\label{sec:neural_cca}
Recently, CCA has been used to aid in interpreting the representations of deep neural networks~\citep{raghu2017svcca,morcos2018insights,kornblith2019similarity}. These approaches are restricted to layer-wise comparisons of representations, reduced-dimensionality views of representations (via PCA), or small dataset sizes to accomodate current limits of CCA approaches. We replicate one of their analyses (specifically Fig. 1a of~\citep{morcos2018insights}) on the activations of an entire network (not just a layer), unblocking this type of analysis for larger deep learning models.

The largest dimensions handled in~\citep{raghu2017svcca} are $\mathcal{O}(10^3)$. Figure~\ref{fig:cifar10_analysis} demonstrates our approach (parallelized over $8$ TPU chips) on $\mathcal{O}(10^3)$ dimensions (left), $\mathcal{O}(10^5)$ dimensions (middle), and $\mathcal{O}(10^6)$ dimensions (right). Note that in these experiments, we are loading minibatches of CIFAR-10 images, running them through a deep convolutional network, harvesting the activations, and then passing them to our distributed \sgeg{} solver. As mentioned in Section~\ref{sec:game}, our understanding of the geometry of the utilities suggests replacing the standard gradient ascent on $\hat{v}_i$ with Adam~\citep{kingma2014adam}; Adam exhibits behavior that implicitly improves stability around equilibria~\citep{gemp2019unreasonable}. For the smaller $\mathcal{O}(10^3)$ setting, where we can exactly compute ground truth using \texttt{scipy}, we confirm that our approach converges to the top-$1024$ (out of $2048$ possible) eigenvectors with a subspace error of $0.002$ (see~\Appxref{app:prelims}).

\section{Conclusion}\label{sec:conc}

We presented \geg{}, a game-theoretic formulation of the generalized eigenvalue problem (\gep{}). Our formulation enabled the development of a novel algorithm that scales to massive streaming datasets. The \gep{} underlies many classical data processing tools across the sciences, and we believe our proposed approach unblocks its use on the ever-growing size of modern datasets in the streaming setting. In particular, it achieves this by parallelizing computation using modern AI-centric distributed compute infrastructure such as GPUs and TPUs.

\paragraph{Acknowledgements.}
We thank Claire Vernade for reviewing the paper and proofs and early discussions. Thore Graepel for early discussions. Zhe Wang for guidance on engineering. Sukhdeep Singh for project management.

\bibliography{bib}
\bibliographystyle{iclr2023_conference}

\newpage
\onecolumn
\begin{appendices}

\section{Preliminaries: Generalized Eigenvalue Problem}\label{app:prelims}

The following known properties of the \gep{} are useful for our analysis and broadening the scope of \gep{} applications.

\begin{lemma}[$B$-orthogonality]
\label{b_orth}
$v_i^\top B v_j = v_j^\top B v_i = 0$ for any distinct pair of generalized eigenvectors of $Av = \lambda Bv$ where $A$ is symmetric and $B$ is symmetric positive definite.
\end{lemma}
\begin{proof}
Consider the eigenvalue problem $B^{-\frac{1}{2}} A B^{-\frac{1}{2}} w = \lambda w$. Let $v$ be a generalized eigenvector of the generalized eigenvalue problem $Av = \lambda' Bv$. Then the former eigenvalue problem is solved by $w = B^{\frac{1}{2}} v$. By inspection, $B^{-\frac{1}{2}} A B^{-\frac{1}{2}} w = B^{-\frac{1}{2}} A B^{-\frac{1}{2}} B^{\frac{1}{2}} v = B^{-\frac{1}{2}} A v = \lambda' B^{-\frac{1}{2}} B v = \lambda' B^{\frac{1}{2}} v = \lambda' w$. Direct computation of the Rayleigh quotients for both problems reveals $\lambda = \lambda'$. Note that $B$ is positive definite, i.e., full-rank, establishing a bijection between $v$ and $w$: $v = B^{-\frac{1}{2}} w$.  Also, note that $B^{-\frac{1}{2}} A B^{-\frac{1}{2}}$ is symmetric because $A$ and $B$ are symmetric, therefore, $w$ may be chosen such that $W^\top W = I$ which implies $w_i^\top w_j = \delta_{ij} = v_i^\top B^{\frac{1}{2}} B^{\frac{1}{2}} v_j = v_i^\top B v_j$, i.e., the generalized eignvectors are $B$-orthogonal.
\end{proof}

\begin{proposition}[Similar Matrices]
\label{sim_mat}
Given symmetric matrices $A$ and $B \succ 0$, consider the generalized eigenvalue problem $Av = \lambda' Bv$ with $\lambda'$ and $v$ its corresponding generalized eigenvalues and eigenvectors. Then the eigenvectors and eigenvalues of the related eigenvalue problem $B^{-\frac{1}{2}} A B^{-\frac{1}{2}} w = \lambda w$ are $w=B^{\frac{1}{2}}v$ and $\lambda=\lambda'$.
\end{proposition}
\begin{proof}
    The relationship between the eigenvectors of the two problems is proven in Lemma~\ref{b_orth}. The relationship between the eigenvalues can be proven by inspection after calculating the Rayleigh quotients for both problems:
    \begin{align}
        \lambda' &= \frac{v^\top A v}{v^\top B v}
        \\ \lambda &= \frac{w^\top B^{-\frac{1}{2}} A B^{-\frac{1}{2}} w}{w^\top w}
        \\ &= \frac{v^\top B^{\frac{1}{2}} B^{-\frac{1}{2}} A B^{-\frac{1}{2}} B^{\frac{1}{2}} v}{v^\top B v}
        \\ &= \frac{v^\top A v}{v^\top B v}
        \\ &= \lambda'.
\end{align}
\end{proof}

\subsection{Computing Subspace Error for \gep{}}

Lemma~\ref{b_orth} states that the generalized eigenvectors are $B$-orthogonal rather than orthogonal under the standard Euclidean basis. Therefore, we cannot compute subspace error in the same way as is typically done for e.g., singular value decomposition. However, we can exploit Lemma~\ref{sim_mat} to compute subspace error for the related eigenvalue problem $B^{-\frac{1}{2}} A B^{-\frac{1}{2}} w = \lambda w$ which \emph{does} have orthogonal eigenvectors due to its symmetry.

Formally, let $v$ be a solution to the \gep{}, $Av = \lambda' Bv$. Then by Lemma~\ref{sim_mat}, $w = B^{1/2} v$ is a solution to $B^{-1/2} A B^{-1/2} w = \lambda w$, with eigenvalue $\lambda = \lambda'$. Leveraging this equivalence, we can measure subspace error of the \gep{} solution by first mapping it to the normalized case and computing subspace error there where $W$ contains the top-$k$ eigenvectors of $B^{-1/2} A B^{-1/2}$. Also let $\hat{W} = B^{1/2} \hat{V}$ where $\hat{V}$ contains our top-$k$ approximations. Given the top-$k$ ground truth eigenvectors $W$ and approximations $W$, normalized subspace error can then be computed as $1 - \frac{1}{k} \trace(U^* P) \in [0, 1]$ where $U^* = W W^\dagger$ and $P = \hat{W} \hat{W}^\dagger$~\citep{gemp2021eigengame,tang2019exponentially}.

\subsection{Courant-Fischer Min-Max Principle}\label{cf_minmax}

The Courant-Fischer Min-Max principle states that the $i$th largest generalized eigenvalue is given by the \emph{minimum} possible Rayleigh quotient within the $i$-dimensional subspace $S$ that captures \emph{maximal} trace~\citep{avron2008generalized,parlett1998symmetric}:
\begin{align}
    v_i &= \argmin_{v_i \in S} \max_{\text{dim}(S)=i} \frac{v_i^\top A v_i}{v_i^\top B v_i}.
\end{align}

Note this defines each eigenvalue of the \gep{} as the value (at Nash equilibrium) of a min-max (two-player, zero-sum) game rather than the entire set of top-$k$ eigenvectors/eigenvalues as the Nash equilibrium / utility-at-Nash of a $k$-player, general-sum game.
\section{\geg{} is Well-Posed}\label{app:well_posed}

First, we prove \geg{} suitably captures the top-$k$ \gep{}.

\begin{replemma}{well_posed_utils}[Well-posed Utilities]
Given exact parents and assuming the top-$k$ eigenvalues of $B^{-1}A$ are distinct and positive, the maximizer of player $i$'s utility is the unique generalized eigenvector $v_i$ (up to sign, i.e., $-v_i$ is also valid).
\end{replemma}
\begin{proof}
\textcolor{highlight}{\emph{Approach}: We will represent each $\hat{v}_i$ as a linear combination of the true eigenvectors, $v_p$ for $p \in \{1, d\}$. We will then show that maximizing the utility for each player with exact parents is equivalent to solving a linear program. This resulting problem has a unique solution, which is the true eigenvector $v_i$.}

Assume the parents have been learned exactly and let $\hat{v}_i = \sum_p w_p v_p$ with $||\hat{v}_i|| = 1$ and where $w_p$ are the weights of the linear combination. Expand and simplify the following expressions that appear in the utility definition with the knowledge that the generalized eigenvectors are guaranteed to be $B$-orthogonal, i.e., $v_i^\top B v_j = 0$ for all $i \ne j$ (see Lemma~\ref{b_orth} in appendix):

\begin{align}
    \langle \hat{v}_i , B \hat{v}_i \rangle &= (\sum_p w_p v_p)^\top B (\sum_l w_l v_l) = \sum_p \sum_l w_p w_l v_p^\top B v_l = \sum_p w_p^2 \langle v_p, B v_p \rangle
    \\ \langle \hat{v}_i , A \hat{v}_i \rangle &= (\sum_p w_p v_p)^\top A (\sum_l w_l v_l) = \sum_p \sum_l \lambda_l w_p w_l v_p^\top B v_l = \sum_p \lambda_p w_p^2 \langle v_p, B v_p \rangle
    \\ \langle \hat{v}_i , B v_j \rangle &= (\sum_p w_p v_p)^\top B v_j = w_j \langle v_j, B v_j \rangle
    \\ \langle \hat{v}_i , A v_j \rangle &= (\sum_p w_p v_p)^\top A v_j = \lambda_j  w_j \langle v_j, B v_j \rangle.
\end{align}

Plugging these in to the utility function, we find
\begin{align}
    u_i(\hat{v}_i \vert v_{j<i}) &= \frac{\langle \hat{v}_i, A \hat{v}_i \rangle}{\langle \hat{v}_i, B \hat{v}_i\rangle} - \sum_{j < i} \frac{\text{\textcolor{blue}{$\langle v_j, A v_j \rangle$}} \langle \hat{v}_i, B v_j \rangle^2}{\langle v_j, B v_j \rangle^2 \langle \hat{v}_i, B \hat{v}_i \rangle}
    \\ &= \frac{\langle \hat{v}_i, A \hat{v}_i \rangle}{\langle \hat{v}_i, B \hat{v}_i\rangle} - \sum_{j < i} \frac{\text{\textcolor{green}{$\lambda_j \langle \hat{v}_i, B v_j \rangle$}}^2}{\langle v_j, B v_j \rangle \langle \hat{v}_i, B \hat{v}_i \rangle} \hspace{1.0cm} \text{\textcolor{blue}{$\langle v_j , A v_j \rangle \rightarrow \langle v_j , \lambda_j B v_j \rangle$}}
    \\ &= \frac{\langle \hat{v}_i, A \hat{v}_i \rangle}{\langle \hat{v}_i, B \hat{v}_i\rangle} - \sum_{j < i} \frac{\langle \hat{v}_i, A v_j \rangle \langle \hat{v}_i, B v_j \rangle}{\langle v_j, B v_j \rangle \langle \hat{v}_i, B \hat{v}_i \rangle} \hspace{1.0cm} \text{\textcolor{green}{$\langle \hat{v}_i , \lambda_j B v_j \rangle \rightarrow \langle \hat{v}_i , A v_j \rangle$}}
    \\ &= \frac{1}{\sum_p w_p^2 \langle v_p, B v_p \rangle} \Big[ \sum_l \lambda_l w_l^2 \langle v_l, B v_l \rangle - \sum_{j < i} \frac{(\lambda_j  w_j \langle v_j, B v_j \rangle)(w_j \cancel{\langle v_j, B v_j \rangle})}{\cancel{\langle v_j, B v_j \rangle}} \Big]
    \\ &= \sum_l \lambda_l z_l - \sum_{j < i} \lambda_j  z_j = \sum_{j \ge i} \lambda_j  z_j.
\end{align}

where
\begin{align}
    z_j &= \frac{w_j^2 \langle v_j, B v_j \rangle}{\sum_p w_p^2 \langle v_p, B v_p \rangle} = \frac{w_j^2 b_j^2}{\sum_p w_p^2 b_p^2} = \frac{q_j^2}{\sum_p q_p^2} \label{zq}
\end{align}
and $z \in \Delta^{d-1}$.

This is a linear optimization problem over the simplex. Given that the eigenvalues are distinct and positive, we have that the unique solution is $z = e_i$, the onehot vector with a $1$ at index $i$.

In order to prove uniqueness of $w$ (up to sign), we apply Lemma~\ref{simplex_sphere_bijection}, which proves a bijection (up to sign) between $z$ and $w$, completing the proof.
\end{proof}

\begin{lemma}
\label{simplex_sphere_bijection}
Let $z \in \Delta^{d-1}$ such that $z_j = \frac{w_j^2 \langle v_j, B v_j \rangle}{\sum_p w_p^2 \langle v_p, B v_p \rangle}$ where $w$ parameterizes the approximation $\hat{v}_i = \sum_p w_p v_p \in \mathcal{S}^{d-1}$. There exists a unique bijection (up to sign of $w_j$) between $z_j$ and $w_j$, i.e., $w_j = \pm g(z)_j$.
\end{lemma}
\begin{proof}
Let $b_j = \langle v_j, B v_j \rangle$ and $q_j = w_j b_j$ so that $w_j = q_j / b_j$. Then $\hat{v}_i = \sum_p \frac{q_p}{b_p} v_p$. Also, $q_j^2 = c z_j$ where $c = \sum_p q_p^2$ so that $q_j^2$ is uniquely defined up to a scalar multiple, i.e., its direction is immediately unique by this formula but not its magnitude. Recall $\langle v_i, B v_j \rangle = 0$ for all $i \ne j$ which implies $V^\top B V$ is diagonal. Therefore, the constraint $||\hat{v}_i|| = ||w||^2_{V^\top V} = 1$ translates to $||\frac{q}{b}||^2_{V^\top V} = q^\top (V^\top B V)^{-1/2} (V^\top V) (V^\top B V)^{-1/2} q = q^\top D^{-1} q = 1$. In other words, whereas an approximate eigenvector for the standard eigenvalue problem can be modeled as choosing a vector on the unit-sphere, an approximate eigenvector for the generalized eigenvalue problem is modeled as choosing a vector on an ellipsoid ($D$ is positive definite because $V^\top V$ is symmetric positive definite assuming distinct eigenvalues, and we are given $B$ is symmetric positive definite). This result uniquely defines a magnitude for $q$, therefore, combining it with the previous result uniquely defines $w_j^2$ from $q_j^2$ completing the bijection. The only degree of freedom that remains is the sign of $w_j$ which is expected as both $v_i$ and $-v_i$ are valid eigenvectors.
\end{proof}

Although player $i$'s utility $u_i$ appears abstruse, it actually has a simple explanation and structure.

\begin{repproposition}{util_shape}[Utility Shape]
Each player's utility is periodic in the angular deviation ($\theta$) along the sphere. Its shape is sinusoidal, but with its angular axis ($\theta$) smoothly deformed as a function of $B$. Most importantly, every local maximum is a global maximum (see Figure~\ref{fig:warped_sinusoid} for an example).
\end{repproposition}
\begin{proof}
Lemma~\ref{well_posed_utils} proves each utility function can be represented as a linear function over a simplex $z \in \Delta^{d-1}$. Lemma~\ref{simplex_sphere_bijection} then proves this simplex can be parameterized by a variable $q$ constrained to an ellipsoid with curvature $D = \texttt{diag}(\ldots, \langle v_i, B v_i \rangle, \ldots)$. This matches the analysis of EigenGame exactly, except that $D=I$ in that previous work. The implication is that each utility function as defined in~\eqref{eqn:appx:util_new} is also a cosine, but with its angular axis deformed according to $D$.
\end{proof}

As an example consider setting
\begin{align}
    A &= \begin{bmatrix} 0.77759061 & 0.26842584
    \\ 0.26842584 & 0.87788983
    \end{bmatrix}
    & B &= \begin{bmatrix} 0.2325605 & 0.06042127
    \\ 0.06042127 & 0.03241424
    \end{bmatrix}
\end{align}
and observe the utilities in Figure~\ref{fig:warped_sinusoid}.

Our proposed utilities tie nicely back to previous work~(\cite{gemp2022eigengame}, Appx. J.2) via their gradients and our derived update directions.

\begin{lemma}[\geg{} Gradient]\label{lemma:grad_derivation}
The gradient of player $i$'s utility with respect to $\hat{v}_i$ is
\begin{align}
    2 \times \Big[ \frac{(\hat{v}_i^\top B \hat{v}_i) A \hat{v}_i - (\hat{v}_i^\top A \hat{v}_i) B \hat{v}_i}{\langle \hat{v}_i, B \hat{v}_i \rangle^2} - \sum_{j < i} \frac{\hat{\lambda}_j}{\langle \hat{v}_j, B \hat{v}_j \rangle} (\hat{v}_i^\top B \hat{v}_j) \frac{\big[ \langle \hat{v}_i, B \hat{v}_i \rangle B \hat{v}_j - \langle \hat{v}_i, B \hat{v}_j \rangle B \hat{v}_i \big]}{\langle \hat{v}_i, B \hat{v}_i \rangle^2} \Big].
\end{align}
\end{lemma}
\begin{proof}
Recall player $i$'s utility function:
\begin{align}
    u_i(\hat{v}_i \vert \hat{v}_{j<i}) &= 
    \underbrace{\hat{\lambda}_i}_{\text{\textcolor{blue}{reward}}} - \sum_{j < i} \underbrace{\hat{\lambda}_j \langle \textcolor{green}{\hat{y}_i}, B \textcolor{green}{\hat{y}_j} \rangle^2}_{\text{\textcolor{red}{penalty}}} \quad \text{ where $\textcolor{green}{\hat{y}_i} = \frac{\hat{v}_i}{||\hat{v}_i||_B}$,}
\end{align}
$\hat{\lambda}_i = \frac{\langle \hat{v}_i, A \hat{v}_i \rangle}{\langle \hat{v}_i, B \hat{v}_i \rangle}$, and $||z||_B = \sqrt{\langle z, B z \rangle}$.

We will address the gradient of each term in the chain rule in sequence. First consider $\hat{\lambda}_i$:
\begin{align}
    \nabla_{\hat{v}_i} \hat{\lambda}_i &= \nabla_{\hat{v}_i} \Big\{ \frac{\langle \hat{v}_i, A \hat{v}_i \rangle}{\langle \hat{v}_i, B \hat{v}_i \rangle} \Big\} = \nabla_{\hat{v}_i} \Big\{ \langle \hat{v}_i, A \hat{v}_i \rangle \langle \hat{v}_i, B \hat{v}_i \rangle^{-1} \Big\}
    \\ &= \frac{2 (\hat{v}_i^\top B \hat{v}_i) A \hat{v}_i - 2 (\hat{v}_i^\top A \hat{v}_i) B \hat{v}_i}{\langle \hat{v}_i, B \hat{v}_i \rangle^2}.
\end{align}

The next term that depends on $\hat{v}_i$ is $\langle \hat{y}_i, B \hat{y}_j \rangle^2$ through $\hat{y}_i$:
\begin{align}
    \nabla_{\hat{v}_i} \langle \hat{y}_i, B \hat{y}_j \rangle^2 &= \nabla_{\hat{v}_i} \Big\{ \langle \hat{v}_i, B \hat{y}_j \rangle^2 \langle \hat{v}_i, B \hat{v}_i \rangle^{-1} \Big\}
    \\ &= 2 \langle \hat{v}_i, B \hat{y}_j \rangle B \hat{y}_j \langle \hat{v}_i, B \hat{v}_i \rangle^{-1} - \langle \hat{v}_i, B \hat{y}_j \rangle^2 \langle \hat{v}_i, B \hat{v}_i \rangle^{-2} (2 B \hat{v}_i)
    \\ &= \frac{2 \langle \hat{v}_i, B \hat{y}_j \rangle \big( \langle \hat{v}_i, B \hat{v}_i \rangle B \hat{y}_j - \langle \hat{v}_i, B \hat{y}_j \rangle B \hat{v}_i \big)}{\langle \hat{v}_i, B \hat{v}_i \rangle^{2}}
    \\ &= \frac{2 \langle \hat{v}_i, B \hat{v}_j \rangle \big( \langle \hat{v}_i, B \hat{v}_i \rangle B \hat{v}_j - \langle \hat{v}_i, B \hat{v}_j \rangle B \hat{v}_i \big)}{\langle \hat{v}_i, B \hat{v}_j \rangle \langle \hat{v}_i, B \hat{v}_i \rangle^{2}}
\end{align}
where we have replaced all $\hat{y}_j$ terms with $\frac{\hat{v}_j}{\langle \hat{v}_j, B \hat{v}_j \rangle^{1/2}}$ terms and consolidated the denominators.

Combining these intermediate results, we find
\begin{align}
    \nabla_{\hat{v}_i} u_i &= 2 \Big[ \frac{(\hat{v}_i^\top B \hat{v}_i) A \hat{v}_i - (\hat{v}_i^\top A \hat{v}_i) B \hat{v}_i}{\langle \hat{v}_i, B \hat{v}_i \rangle^2} - \sum_{j < i} \frac{\hat{\lambda}_j}{\langle \hat{v}_j, B \hat{v}_j \rangle} (\hat{v}_i^\top B \hat{v}_j) \frac{\big[ \langle \hat{v}_i, B \hat{v}_i \rangle B \hat{v}_j - \langle \hat{v}_i, B \hat{v}_j \rangle B \hat{v}_i \big]}{\langle \hat{v}_i, B \hat{v}_i \rangle^2} \Big].
\end{align}
\end{proof}

\begin{proposition}[Equivalence to EigenGame Unloaded]
The generalized EigenGame pseudogradient in~\eqref{eq:gevp_update} is equivalent to the Riemannian gradient in~\citep{gemp2021eigengame} when $B=I$.
\end{proposition}
\begin{proof}
In order to compute the Riemannian update direction, we project player $i$'s direction onto the tangent space of the unit-sphere by left-multiplying with $(I - \hat{v}_i \hat{v}_i^\top)$. Starting with~\eqref{eq:gevp_update}, we find
\begin{align}
    \tilde{\nabla}_i &= \overbrace{(\hat{v}_i^\top B \hat{v}_i) A \hat{v}_i - (\hat{v}_i^\top A \hat{v}_i) B \hat{v}_i}^{\text{reward}} - \sum_{j < i} \overbrace{(\hat{v}_i^\top A \hat{y}_j) \big[ \langle \hat{v}_i, B \hat{v}_i \rangle B \hat{y}_j - \langle \hat{v}_i, B \hat{y}_j \rangle B \hat{v}_i}^{\text{penalty}} \big]
    \\ &= A \hat{v}_i - (\hat{v}_i^\top A \hat{v}_i) \hat{v}_i - \sum_{j < i} (\hat{v}_i^\top A \hat{v}_j) \big[ \hat{v}_j - \langle \hat{v}_i, \hat{v}_j \rangle \hat{v}_i \big]
    \\ &= (I - \hat{v}_i \hat{v}_i^\top) [A\hat{v}_i - \sum_{j < i} (\hat{v}_i^\top A \hat{v}_j) \hat{v}_j] = (I - \hat{v}_i \hat{v}_i^\top) \tilde{\nabla}_i^{\mu-EG}.
\end{align}
\end{proof}

\begin{proposition}[\geg{} Utilities as Deflated Rayleigh Quotients]\label{prop:deflation}
\textcolor{highlight}{The generalized EigenGame utilities defined in~\eqref{eqn:appx:util_new} can also be derived from the perspective of maximizing the Rayleigh quotients of a \emph{deflated} matrix assuming exact parents.}
\end{proposition}
\begin{proof}
\emph{Deflating} a matrix means to modify the matrix such that the spectrum corresponding to a certain subspace of the matrix is zero. For example, in the case of the SGEP, the matrix $A$ can be deflated to produce a matrix $\tilde{A} = (I - B \frac{v_j v_j^\top}{||v_j||^2_B}) A$ such that any vector in the span of eigenvector $v_j$ achieves zero eigenvalue:
\begin{align}
    \tilde{A} (w_j v_j) &= w_j (I - B \frac{v_j v_j^\top}{||v_j||^2_B}) A v_j
    \\ &= w_j A v_j - w_j B v_j \frac{v_j^\top A v_j}{||v_j||^2_B}
    \\ &= w_j \lambda_j B v_j - w_j \lambda_j B v_j \frac{v_j^\top B v_j}{||v_j||^2_B} \text{ apply rule $Av_j = \lambda_j Bv_j$}
    \\ &= w_j \lambda_j B v_j (1 - \frac{v_j^\top B v_j}{||v_j||^2_B}) \label{eqn:deflation_step}
    \\ &= 0
\end{align}
where $||v_j||^2_B = v_j^\top B v_j$ and $w_j$ is an arbitrary scalar.

This is useful because it allows us to construct a top-$k$ solver by induction: repeatedly deflate a matrix to ignore the top-$(j<i)$ eigenvectors and then deploy a top-$1$ solver on the deflated matrix to find the $i$th eigenvector. To that end, we can construct the following deflation matrix:
\begin{align}
    \tilde{A}_i &= (I - \sum_{j < i} B \frac{v_j v_j^\top}{||v_j||^2_B}) A.
\end{align}
Note that this definition assumes the parents eigenvectors are exact. If they are approximate, this may not act as a deflation in the precise sense. For example, consider defining $\tilde{A}$ with approximate $\hat{v}_j$ instead of exact $v_j$. Now let $\hat{v}_j = v_1$ for all $j$. If one repeats the analysis above for a vector in the span of $v_1$, they would find that~\eqref{eqn:deflation_step} becomes $w_1 \lambda_1 B v_1 (1 - (i - 1) \frac{v_1^\top B v_1}{||v_1||^2_B}) = w_1 \lambda_1 B v_1 (2 - i)$, i.e., it results in an eigenvalue of $(2 - i) \lambda_1$. This is why the effect of this matrix is more accurately described via penalties. We will clarif this connection next.

With the above preliminaries taken care of, we will now show how to derive our utilities via a deflation perspective. Initially, we will assume exact parents, $\hat{v}_{j<i} = v_{j<i}$.
\begin{align}
    u_i(\hat{v}_i \vert v_{j<i}) &= \frac{\langle \hat{v}_i, \tilde{A}_i \hat{v}_i \rangle}{\langle \hat{v}_i, B \hat{v}_i\rangle}
    \\ &= \frac{\langle \hat{v}_i, (I - \sum_{j < i} B \frac{v_j v_j^\top}{||v_j||^2_B} ) A \hat{v}_i \rangle}{\langle \hat{v}_i, B \hat{v}_i\rangle}
    \\ &= \frac{\langle \hat{v}_i, A \hat{v}_i \rangle}{\langle \hat{v}_i, B \hat{v}_i\rangle} - \sum_{j < i} \frac{\langle \hat{v}_i, B v_j v_j^\top A \hat{v}_i \rangle}{||v_j||^2_B \langle \hat{v}_i, B \hat{v}_i\rangle} \text{ expand sum}
    \\ &= \frac{\langle \hat{v}_i, A \hat{v}_i \rangle}{\langle \hat{v}_i, B \hat{v}_i\rangle} - \sum_{j < i} \frac{\langle \hat{v}_i, B v_j \rangle \langle \hat{v_i}, A v_j \rangle}{\langle v_j, B v_j \rangle \langle \hat{v}_i, B \hat{v}_i\rangle} \text{ split \& transpose inner product}
    \\ &= \frac{\langle \hat{v}_i, A \hat{v}_i \rangle}{\langle \hat{v}_i, B \hat{v}_i\rangle} - \sum_{j < i} \frac{\lambda_j \langle \hat{v}_i, B v_j \rangle^2}{\langle v_j, B v_j \rangle \langle \hat{v}_i, B \hat{v}_i\rangle} \text{ apply rule $Av_j = \lambda_j Bv_j$}
    \\ &= 
    \underbrace{\hat{\lambda}_i}_{\text{\textcolor{blue}{reward}}} - \sum_{j < i} \underbrace{\lambda_j \langle \textcolor{green}{\hat{y}_i}, B \textcolor{green}{y_j} \rangle^2}_{\text{\textcolor{red}{penalty}}} \quad \text{ where $\textcolor{green}{\hat{y}_i} = \frac{\hat{v}_i}{||\hat{v}_i||_B}$,} \label{eqn:appx:util_new}
\end{align}
$\hat{\lambda}_i = \frac{\langle \hat{v}_i, A \hat{v}_i \rangle}{\langle \hat{v}_i, B \hat{v}_i \rangle}$, and $||z||_B = \sqrt{\langle z, B z \rangle}$.

If we then relax our assumption and allow $v_j$ to be approximate ($v_j \rightarrow \hat{v}_j$) we recover our utilities in~\eqref{eqn:appx:util_new}.
\end{proof}
\section{Smooth and Unbiased}\label{app:smooth_unbiased}

In order to prove asymptotic convergence of \sgeg{} in the deterministic setting, we establish the following lemmas.

\begin{lemma}\label{lemma:smooth}
The update $\tilde{\nabla}_i$ in~\eqref{eq:update} is smooth.
\end{lemma}

\begin{proof}
The reward terms are polynomial in $\hat{v}_i$ and therefore smooth (analytic). The numerators of the penalty terms are also polynomial in $\hat{v}_i$ and $\hat{v}_j$, however, the denominator includes a scalar $\langle \hat{v}_j, B \hat{v}_j \rangle$. Given $B \succ 0$, this term is guaranteed to be greater than the minimum eigenvalue of $B$ (which is positive), thereby ensuring the penalty terms are non-singular. So these terms are also smooth.
\end{proof}

Instead of proving the following lemmas directly for \Algref{alg:det_geg} (the deterministic variant), we prove them for \Algref{alg:gen_eg}, which subsumes \Algref{alg:det_geg} ($\rho=0, \gamma_t=1$, $b=n$, $M=1$).

The following two lemmas are proven in the single proof below. Note Lemma~\ref{lemma:right_fp} is essentially a restatement of Lemma~\ref{steepest_ascent} from the main body.
\begin{lemma}\label{lemma:right_fp}
The unique stable fixed point (up to sign of $\hat{v}_j$) of \Algref{alg:gen_eg} run with exact expectations (e.g., $n'=n$ where $n$ is the full dataset size) is $\hat{v}_j = v_j$ for all $j \in \{1, \ldots, k\}$, i.e., the top-$k$ generalized eigenvectors.
\end{lemma}
\begin{lemma}
\Algref{alg:gen_eg}'s updates are asymptotically unbiased.
\end{lemma}
\begin{proof}
The proof is constructed sequentially by proving each update process has a unique stable fixed point conditioned on the previous updates' fixed points defined by the hierarchy imposed on the players. We explain how we are able to address constructing unbiased estimates of each update as well, thereby supporting a stochastic, asymptotic convergence proof.

The proof begins by considering the updates of the first player on $\hat{v}_1$. Player $1$ is unique in that it pays no penalties for aligning with other players. Its update consists of the reward terms only, which comprises an unbiased estimate assuming independent, unbiased estimates for $A$ and $B$ (i.e., these are constructed with independent minibatches). Player $1$'s update simply performs Riemannian gradient ascent on its utility function. Proposition~\ref{util_shape} proves that every maximum of this function is a global maximum (in addition, it contains no saddle points). Therefore, the only stable fixed point for $\hat{v}_1$ is $v_1$.

Next, we consider player $1$'s update to $[B\hat{v}]_1$. Given we just showed $\hat{v}_1$'s stable fixed point is $v_1$, and this update is simply a running average, its unique stable fixed point is $Bv_1$.

We now consider player $2$'s update, which includes penalty terms. Plugging $[B\hat{v}]_1$'s unique stable fixed point into these penalty terms, and again assuming independent, unbiased estimates for $A$ and $B$, allows us to construct an unbiased estimate of the penalty terms. Similarly to player $1$'s analysis, player $2$'s update performs Riemannian gradient ascent on a utility function with a unique stable fixed point, $\hat{v}_2 = v_2$.

The proof then proceeds repeating the same arguments, alternating between proving the unique stable fixed point of each $[B\hat{v}]_j = Bv_j$ and $\hat{v}_j = v_j$.
\end{proof}

\section{Error Propagation}\label{app:err_prop}

An error propagation analysis is necessary to rule out the scenario where an arbitrary (unbounded) level of precision is required by the parents to ensure any progress towards the true solution can be made by the children. In other words, we show that as the parents near their true solutions, the children may near theirs as well.

As before in \Appxref{app:smooth_unbiased}, we will perform this analysis for the stochastic version of the algorithm (\Algref{alg:gen_eg}), but note that this subsumes the analysis for the deterministic version (where $[B\hat{v}]_j$ is replaced by the exact $Bv_j$ with zero error).

Player $1$'s update of $\hat{v}_1$ is unbiased without any assumptions on the state of any of the other player vectors $\hat{v}_{j>1}$ or auxiliary variables $[B\hat{v}]_{j\ge1}$. We would like to understand how transient error in $\hat{v}_1$ propagates through to these other variables. The updates of player $1$'s children depend on $[B\hat{v}]_{1}$, so we analyze the effect on it first. Note that the error propagation analysis naturally repeats as we progress down the hierarchy of players, so we analyze how error in $\hat{v}_{i}$ propagates through to $[B\hat{v}]_{i}$ and then onto $\hat{v}_{j>i}$. Interestingly, step~\ref{errprop:step2} of the following proof suggests the error in the $\hat{v}_i$ must fall below $\frac{1}{\kappa}$ before any increase in accuracy of the parents helps to improve accuracy in the children. This result mirrors that of~\citep{gemp2021eigengame} (see their Appendix F).

\begin{theorem}\label{thm:error_prop}
An $\mathcal{O}(\epsilon)$ angular error in the parent propagates to an $\mathcal{O}(\epsilon^{\frac{1}{2}})$ upper bound on the angular error of the child's solution.
\end{theorem}

\begin{proof}
The proof proceeds in three steps:
\begin{enumerate}
    \item $\mathcal{O}(\epsilon)$ angular error of parent $v_i$ $\implies$ $\mathcal{O}(\epsilon)$ Euclidean error of parent $v_i$ \label{errprop:step1}
    \item $\mathcal{O}(\epsilon)$ Euclidean error of parent $v_i$ $\implies$ $\mathcal{O}(\epsilon)$ Euclidean error of norm of child $v_{j>i}$'s gradient \label{errprop:step2}
    \item $\mathcal{O}(\epsilon)$ Euclidean error of norm child $v_{j>i}$'s gradient + instability of minima at $v_{k \ne j}$ $\implies$ $\mathcal{O}(\epsilon)$ angular error of child $v_{j>i}$'s solution assuming $B=I$. \label{errprop:step3}
    \item $\mathcal{O}(\epsilon)$ angular error of child $v_{j>i}$'s solution assuming $B=I$ $\implies$ $\mathcal{O}(\epsilon^{\frac{1}{2}})$ angular error of child $v_{j>i}$'s solution for a general $B \succ 0$. \label{errprop:step4}
\end{enumerate}

\textbf{\ref{errprop:step1}}. As in $\mu$-EigenGame, an \emph{angular error} of $\mathcal{\epsilon}$ in the parent translates to $\mathcal{\epsilon}$ \emph{Euclidean error}. The proof is exactly the same as in~\citep{gemp2021eigengame}, repeated here for convenience. Angular error in the parent can be converted to Euclidean error by considering the chord length between the mis-specified parent and the true parent direction. The two vectors plus the chord form an isoceles triangle with the relation that chord length $l = 2 \sin(2\epsilon)$ is $\mathcal{O}(\epsilon)$ for $\epsilon \ll 1$.

\textbf{\ref{errprop:step2}}. Next, write the mis-specified parents as $\hat{v}_i = v_i + w_i$ where $||w_i||$ is $\mathcal{O}(\epsilon_i)$ as we just explained.

Now consider the fixed point of the auxiliary variable's update: $[Bv]_i = B (v_i + w_i) = B v_i + B w_i$. Hence any mis-specification in the parent $\hat{v}_i$ appears as a mis-specification of the auxiliary variable's fixed point by $B w_i$, which is $\mathcal{O}(\lambda_{max} \epsilon)$ where $\lambda_{max}$ is the maximum eigenvalue (spectral radius) of $B$. Assume the auxiliary variable is mis-specified by an additional error ($q_i$ where $||q_i||$ is $\mathcal{O}(\epsilon'_i)$) representing failure to precisely reach the perturbed fixed point $B(v_i + w_i)$, i.e., $[B\hat{v}]_i = B(v_i + w_i) + q_i$.

The auxiliary variable impacts the update of $\hat{v}_{j>i}$ through $\hat{y}_i$ and similarly $[B\hat{y}]_i$:
\begin{align}
    \hat{y}_i &= \frac{\hat{v}_i}{\sqrt{\lfloor \langle \hat{v}_i, [B\hat{v}]_i\rangle \rfloor_{\rho}}}
    \\ &= \frac{v_i + w_i}{\sqrt{\lfloor \langle \hat{v}_i, Bv_i\rangle + \langle \hat{v}_i, B w_i \rangle + \langle \hat{v}_i, q_i \rangle \rfloor_{\rho}}}
    \\ &= \frac{v_i + w_i}{\sqrt{\lfloor \langle v_i, Bv_i\rangle + 2 \langle v_i, B w_i \rangle + \langle w_i, Bw_i\rangle + \langle v_i, q_i\rangle + \langle w_i, q_i\rangle \rfloor_{\rho}}}
    \\ &= cy_i + \frac{w_i}{\sqrt{\lfloor \langle v_i, Bv_i\rangle + 2 \langle v_i, B w_i \rangle + \langle w_i, Bw_i\rangle + \langle v_i, q_i\rangle + \langle w_i, q_i\rangle \rfloor_{\rho}}}
    \\ &= c y_i + e_i
\end{align}

In order to bound this term, we make a few mild assumptions.
\begin{itemize}
    \item Assume $\rho$ is less than $\lambda_{min}$ as stated earlier and, in particular, less than the lower bound.
    \item Assume $\epsilon_i'$ is $\mathcal{O}(\epsilon_i)$ to ease the exposition.
    \item Also, w.l.o.g., assume $\lambda_{max} > 1$; if not, we can simply scale the problem such that it is true.
\end{itemize} 

Let $\kappa = \frac{\lambda_{max}}{\lambda_{min}}$ be the condition number of $B$, and note that the error term in the denominator is bounded by the spectrum of $B$:
\begin{align}
    \lfloor \langle \hat{v}_i, [B\hat{v}]_i\rangle \rfloor_{\rho} &\le \langle v_i, B v_i \rangle + 2 \lambda_{max} \epsilon_i + \lambda_{max} \epsilon_i^2 + \epsilon_i' + \epsilon_i \epsilon_i'
    \\ &\overbrace{\le}^{1 \le \lambda_{max} \le \kappa \langle v_i, B v_i \rangle} \langle v_i, B v_i \rangle (1 + 2 \kappa  \epsilon_i + \kappa \epsilon_i^2 + \kappa \epsilon_i' + \kappa \epsilon_i \epsilon_i')
    \\ &\overbrace{\le}^{\epsilon' \text{ is } \mathcal{O}(\epsilon)} \langle v_i, B v_i \rangle(1 + 3 \kappa  \epsilon_i + 2 \kappa \epsilon_i^2)
\end{align}
and vice versa for the lower bound, which implies
\begin{align}
    \\ \lfloor \langle \hat{v}_i, [B\hat{v}]_i\rangle \rfloor_{\rho} &\in \langle v_i, [Bv]_i\rangle \Big[ \big( 1 - 3 \kappa \epsilon_i - 2 \kappa \epsilon_i^2 \big) , \big( 1 + 3 \kappa \epsilon_i + 2 \kappa \epsilon_i^2 \big) \Big].
\end{align}
Then
\begin{align}
    c &\in \Big[ \frac{1}{\sqrt{1 + 3 \kappa \epsilon_i + 2 \kappa \epsilon_i^2}}, \frac{1}{\sqrt{1 - 3 \kappa \epsilon_i - 2 \kappa \epsilon_i^2}} \Big]. \label{c_range}
\end{align}
Note that if $\epsilon_i \ll \frac{1}{\kappa}$, then $c$ is $\mathcal{O}(1)$. Note this condition also implies $e_i$ is $\mathcal{O}(\epsilon_i')$. We will use these facts later.

Now we are prepared to consider the norm of the difference, $d_j$, between the Riemannian\footnote{The Riemannian update projects the vanilla update onto the tangent space of the sphere.} update to $\hat{v}_j$ with exact parents and auxiliary variables versus the actual inexact Riemannian update. Let $\Delta^R_j$ be defined as in line~\ref{eq:gevp_update} of Algorithm~\ref{alg:gen_eg}. Define $\bar{\Delta}_j^R$ to be the same except with $\hat{y}_l$ and $[B\hat{y}]_l$ terms replaced by their true solution counterparts $y_l$ and $[By]_l$. Note that $\Delta^R_j$ and $\hat{\Delta}^R_j$ already live in the tangent space of the unit-sphere at $\hat{v}_j$. Then the norm of the difference between the two update directions is upper bounded as
\begin{align}
    ||d_j|| &= ||\Delta^R_j - \bar{\Delta}^R_j||
    \\ &\le ||\sum_{l < i} \Big[ (\hat{v}_j^\top A \hat{y}_l) \big[ \langle \hat{v}_j, B \hat{v}_j \rangle [B \hat{y}]_l - \langle \hat{v}_j, [B \hat{y}]_l \rangle B \hat{v}_j \big] - \ldots \Big]||
    \\ &\le \sum_{l < i} ||\Big[ (\hat{v}_j^\top A \hat{y}_l) \big[ \langle \hat{v}_j, B \hat{v}_j \rangle [B \hat{y}]_l - \langle \hat{v}_j, [B \hat{y}]_l \rangle B \hat{v}_j \big] - \ldots \Big]||. \label{summand}
\end{align}

Recall $q_i$ is the error associated with suboptimality of $[B\hat{v}]_i$ and propoagates to $[B\hat{y}]_i$ as defined on line~\ref{eq:Byj} of Algorithm~\ref{alg:gen_eg}. Let
\begin{align}
    p_i &= \frac{q_i}{\sqrt{\lfloor \langle v_i, Bv_i\rangle + 2 \langle v_i, B w_i \rangle + \langle w_i, Bw_i\rangle + \langle v_i, q_i\rangle + \langle w_i, q_i\rangle \rfloor_{\rho}}}.
\end{align}
By similar arguments used to bound $e_i$, $||p_i||$ is $\mathcal{O}(\epsilon'_i)$ if $\epsilon_i \ll \frac{1}{\kappa}$.

Bounding the summand in~\eqref{summand}, we find
\begin{align}
    &|| (\hat{v}_j^\top A \hat{y}_l) \big[ \langle \hat{v}_j, B \hat{v}_j \rangle [B \hat{y}]_l - \langle \hat{v}_j, [B \hat{y}]_l \rangle B \hat{v}_j \big]
    -
    (\hat{v}_j^\top A y_l) \big[ \langle \hat{v}_j, B \hat{v}_j \rangle [B y]_l - \langle \hat{v}_j, [B y]_l \rangle B \hat{v}_j \big] ||
    \\ &= || (c \hat{v}_j^\top A y_l + \hat{v}_j^\top A e_l) \big[ \langle \hat{v}_j, B \hat{v}_j \rangle (c By_l + Be_l + p_l) - \langle \hat{v}_j, (c By_l + Be_l + p_l) \rangle B \hat{v}_j \big]
    -
    \ldots ||
    \\ &= || (c^2 - 1) (\hat{v}_j^\top A y_l) \big[ \langle \hat{v}_j, B \hat{v}_j \rangle By_l - \langle \hat{v}_j, By_l \rangle B \hat{v}_j \big]
    + \mathcal{O}(\epsilon) ||.
\end{align}

Recall~\eqref{c_range} and note that
\begin{align}
    c^2 - 1 &\in \{ \frac{-3 \kappa \epsilon_i - 2 \kappa \epsilon_i^2}{1 + 3 \kappa \epsilon_i + 2 \kappa \epsilon_i^2}, \frac{3 \kappa \epsilon_i + 2 \kappa \epsilon_i^2}{1 - 3 \kappa \epsilon_i - 2 \kappa \epsilon_i^2} \}
\end{align}
which has norm $\vert c^2 - 1 \vert = \mathcal{O}(\kappa \epsilon_i)$. Therefore, taking into account the impact of the other $A$ and $B$ terms, $||d_j||$ is upper bounded by $\mathcal{O}(i\kappa \sigma(A) \lambda_{\max}^2 \epsilon_i)$ where $\sigma(A)$ is the spectral radius of $A$.

\begin{figure}[!ht]
    \centering
    \begin{subfigure}[t]{.25\columnwidth}
    \centering
    \includegraphics[width=\textwidth]{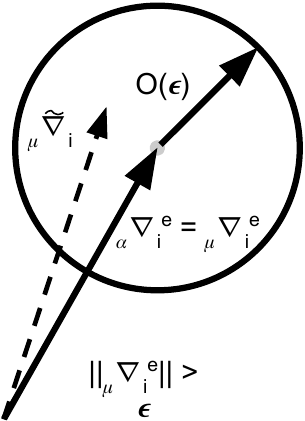}
    \caption{\label{l2tocos}}
    \end{subfigure}
    \hspace{0.2cm}
    \begin{subfigure}[t]{.72\columnwidth}
    \centering
    \includegraphics[width=\textwidth]{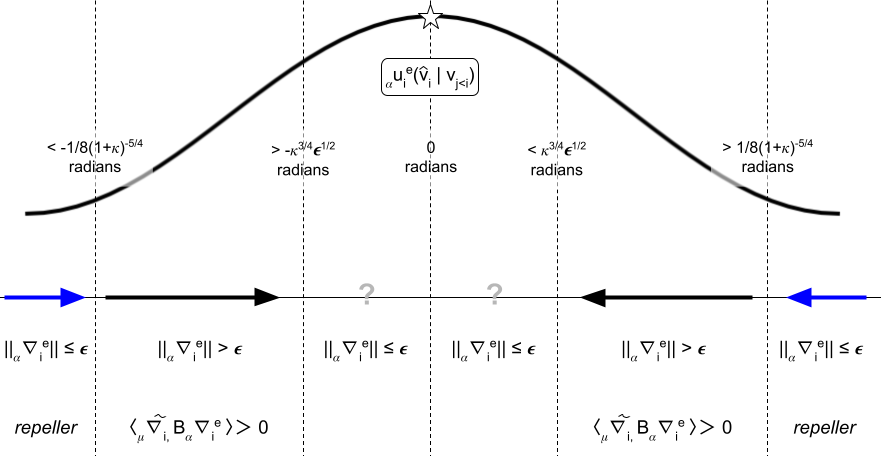}
    \caption{\label{costrap}}
    \end{subfigure}
    \caption{(\subref{l2tocos}) Close in Euclidean distance can imply close in angular distance if the vectors are long enough (reprinted with permission from~\citep{gemp2021eigengame}). (\subref{costrap}) The stable region may consist of an $\mathcal{O}(\kappa^{\frac{3}{4}}\epsilon^{\frac{1}{2}})$ ball around the true optimum as $\epsilon \rightarrow 0$.}
    \label{fig:errprop}
\end{figure}

\textbf{\ref{errprop:step3}}. We can reuse the analysis of~\citep{gemp2021eigengame} to understand how a change in the norm of the vector field relates to a change in the location of the fixed point. This is because the Riemmanian update direction of our proposed method with exact parents and $B=I$ is equivalent to the Riemannian update direction in~\citep{gemp2021eigengame} (simply left-multiply their equation (4) by ($I- \hat{v}_i \hat{v}_i^\top$) to compute their Riemannian update). Therefore, as in this prior work, an error in the gradient norm translates to the same order of error in angular distance to the true fixed point (inflated by a finite scaling dependent on the spectrum of $B$\textemdash accounted for in Step 4 next).

Also, the region around any generalized eigenvector $v_{l \ne j}$ is unstable and this is because the Riemmanian Hessian at that point is positive (this implies instability because we are maximimizing). We can reason that the Riemmanian Hessian is positive by appealing to the fact that the Riemmanian Hessian of our generalized EVP utilities is related to the Hessian of prior work by a warping defined by the positive definite matrix $B$ (for a visual, see Figure~\ref{fig:warped_sinusoid}; for math, see Lemmas~\ref{well_posed_utils} and~\ref{simplex_sphere_bijection}).

In contrast to this prior work, the generalized eigenvectors are more generally, $B$-orthogonal (see Lemma~\ref{b_orth}). By Lemma~\ref{orth_angle_change}, the angular distance between generalized eigenvectors is finite and depends on the condition number of $B$. Therefore, there exists a small enough $\epsilon$ such that an $\epsilon$-ball around any \emph{unstable} region and an $\epsilon$-ball around the \emph{stable} region no longer overlap.

\textbf{\ref{errprop:step4}}. Lastly, Lemma~\ref{angle_blowup} proves that an $\epsilon_i < 1$ angular error assuming $B=I$ can be increased to at most $\kappa^{\frac{3}{4}} \epsilon^{\frac{1}{2}}_i$ if $B$ is relaxed to be any symmetric positive definite matrix with condition number $\kappa$.
\end{proof}

\begin{lemma}\label{orth_angle_change}
The angle between a pair of orthonormal vectors when instead measured under a general positive definite matrix $C$ is lower bounded by $\frac{1}{8} (1+\kappa)^{-\frac{5}{4}}$ where $\kappa$ is the condition number of $C$, i.e., if $\langle v_i, v_j \rangle = 0$, then $\vert \theta \vert = \arccos \Big( \frac{\vert \langle C^{\frac{1}{2}} v_i, C^{\frac{1}{2}} v_j \rangle \vert}{||C^{\frac{1}{2}} v_i|| ||C^{\frac{1}{2}} v_j||} \Big) > \frac{1}{8} (1+\kappa)^{-\frac{5}{4}}$ radians.
\end{lemma}

\begin{proof}

The angle between two vectors is a function of their relation to each other in the two-dimensional plane defined by their pair. Therefore, without loss of generality, consider two vectors $u = \begin{bmatrix} 1 & 0 \end{bmatrix}$ and $v = \begin{bmatrix} 0 & 1 \end{bmatrix}$ and consider the effect of an arbitrary positive definite matrix $\hat{C}$ on their angle. For ease of exposition, denote $\tau=\kappa^{\frac{1}{2}}$ the condition number of $C^{\frac{1}{2}}$.

Let $\hat{C}^{\frac{1}{2}} = \begin{bmatrix} a & c \\ c & b \end{bmatrix}$ be the unique positive definite square root of $\hat{C}$ where $a$ and $b$ are positive and the determinant $ab - c^2 > \gamma > 0$. We aim to show that the magnitude of the angle between $u$ and $v$ under the generalized inner product $\langle \cdot, \cdot \rangle_C$ is lower bounded by a finite, positive quantity dependent on the properties of $C$.

Consider
\begin{align}
    \hat{C}^{\frac{1}{2}} u &= \begin{bmatrix} a \\ c \end{bmatrix} ,& \hat{C}^{\frac{1}{2}} v &= \begin{bmatrix} c \\ b \end{bmatrix},
    \\ ||\hat{C}^{\frac{1}{2}} u|| &= \sqrt{a^2 + c^2} ,& ||\hat{C}^{\frac{1}{2}} v|| &= \sqrt{c^2 + b^2},
    \\ \langle \hat{C}^{\frac{1}{2}} u, \hat{C}^{\frac{1}{2}} v \rangle &= c (a+b). &&
\end{align}

Then
\begin{align}
    \frac{\langle \hat{C}^{\frac{1}{2}} u, \hat{C}^{\frac{1}{2}} v \rangle}{||C^{\frac{1}{2}} u|| ||C^{\frac{1}{2}} v||} &= \frac{c(a+b)}{\sqrt{(a^2 + c^2) (c^2 + b^2)}}
    \\ &\overbrace{<}^{\text{ratio is inc. in $c$ for $ab > c^2$}} \frac{\sqrt{ab - \gamma}(a+b)}{\sqrt{(a^2 + ab - \gamma) (ab + b^2 - \gamma)}} %
    \\ &\overbrace{=}^{\text{div. num. \& den. by $\frac{1}{a^2}$}} \frac{\sqrt{\frac{b}{a} - \frac{\gamma}{a^2}}(1+\frac{b}{a})}{\sqrt{(1 + \frac{b}{a} - \frac{\gamma}{a^2}) (\frac{b}{a} + (\frac{b}{a})^2 - \frac{\gamma}{a^2})}}
    \\ &\overbrace{\le}^{\text{ratio is inc. in $\frac{b}{a}$}} \frac{\sqrt{\tau - \frac{\gamma}{a^2}}(1+\tau)}{\sqrt{(1 + \tau - \frac{\gamma}{a^2}) (\tau + \tau^2 - \frac{\gamma}{a^2})}}  %
    \\ &\overbrace{<}^{\text{ratio is inc. in $\tau-\frac{\gamma}{a^2}$}} \frac{\sqrt{\tau - \frac{\gamma}{\trace^2}}(1+\tau)}{\sqrt{(1 + \tau - \frac{\gamma}{\trace^2}) (\tau + \tau^2 - \frac{\gamma}{\trace^2})}}. %
\end{align}

Note that $\frac{b}{a}$ is a lower bound on the condition number of $\hat{C}$; $\frac{b}{a}$ is equal to the condition number of $\hat{C}^{\frac{1}{2}}$ when $c=0$ (assuming $b>a$, otherwise, $\frac{b}{a} < \frac{a}{b}$ clearly), and the condition number can only increase as $c$ deviates from $0$. Lastly, note that the condition number of $\hat{C}^{\frac{1}{2}}$ is upper bounded by $\tau$, the condition number of $C^{\frac{1}{2}}$. Recall, by replacing $\frac{b}{a}$ with a larger number $\tau$ and $\frac{\gamma}{a^2}$ with a strictly smaller number $\frac{\gamma}{\trace^2}$, $\frac{b}{a} - \frac{\gamma}{a^2}$ implies that $\tau > \frac{\gamma}{\trace^2}$.

Now consider
\begin{align}
    \Big( \frac{\langle C^{\frac{1}{2}} u, C^{\frac{1}{2}} v \rangle}{||C^{\frac{1}{2}} u|| ||C^{\frac{1}{2}} v||} \Big)^2 &< \frac{(\tau - \frac{\gamma}{\trace^2})(1+\tau)^2}{(1 + \tau - \frac{\gamma}{\trace^2}) (\tau + \tau^2 - \frac{\gamma}{\trace^2})}
    \\ &= \frac{(\tau - \frac{\gamma}{\trace^2})(1+\tau)^2}{(\tau - \frac{\gamma}{\trace^2})(1+\tau)^2 + (\frac{\gamma}{\trace^2})^2}
    \\ &= 1 - \frac{(\frac{\gamma}{\trace^2})^2}{(\tau - \frac{\gamma}{\trace^2})(1+\tau)^2 + (\frac{\gamma}{\trace^2})^2}
    \\ &\le 1 - \frac{(\frac{\gamma}{\trace^2})^2}{(\tau)(1+\tau)^2 + (1 + \tau)^2}
    \\ &= 1 - \frac{(\frac{\gamma}{\trace^2})^2}{(1+\tau)^3}.
\end{align}

We can also simplify the fraction $\frac{\gamma}{\trace^2}$. $\gamma$ is a lower bound on the determinant, which is equal to the product of eigenvalues of $\hat{C}^{\frac{1}{2}}$. This is lower bounded by $\lambda_{min}^2$ for any two-dimensional subspace, therefore, $\lambda_{min}^2 < \gamma$. Furthermore, the trace of the matrix is equal to the sum of its eigenvalues. This is upper bounded by $2\lambda_{max}$ for any two-dimensional subspace, therefore $\frac{\gamma}{\trace^2} > \frac{\lambda_{min}^2}{4\lambda_{max}^2} = \frac{1}{4\tau^2} > \frac{1}{4(\tau + 1)^2}$.

Note that $\vert \theta \vert \ge \vert \sin(\theta) \vert$. Therefore,
\begin{align}
    \vert \theta \vert &\ge \vert \sin(\theta) \vert = \sqrt{1 - \cos^2(\theta)}
    \\ &> \sqrt{\frac{(\frac{\gamma}{\trace^2})^2}{(1+\tau)^3}}
    \\ &> \frac{1}{2}\sqrt{(1+\tau)^{-5}}
    \\ &= \frac{1}{2}(1+\tau)^{-\frac{5}{2}}.
\end{align}

Clearly this bound is loose as it implies $\theta$ is only greater than $2^{-\frac{7}{2}}$ for $C = I$ ($\tau=1$) whereas we know in that case that $\theta = \frac{\pi}{2}$. However, this bound serves its purpose of establishing a finite impact of $C$ on the orthogonality of the original two vectors, i.e., if $u$ and $v$ are orthogonal, their angle as measured under a generalized inner product by $C \succ 0$ cannot be $0$.

Replacing $\tau$ with $\kappa^{\frac{1}{2}}$, we can further simplify the bound to
\begin{align}
    \vert \theta \vert &> \frac{1}{2}(1+\tau)^{-\frac{5}{2}}
    \\ &\ge \frac{1}{8}(1+\kappa)^{-\frac{5}{4}}.
\end{align}
\end{proof}

\begin{lemma}\label{angle_blowup}
The angle between a pair of nearly parallel vectors ($\vert \theta \vert$ is $\mathcal{O}(\epsilon)$ with $\epsilon \ll 1$) when instead measured under a general positive definite matrix $C$ is upper bounded by $\mathcal{O}(\kappa^{\frac{3}{4}} \epsilon^{\frac{1}{2}})$.
\end{lemma}

\begin{proof}
Let $C^{\frac{1}{2}} = \begin{bmatrix} a & c \\ c & b \end{bmatrix}$ be the unique positive definite square root of $C$ where $a$ and $b$ are positive and the determinant $ab - c^2 > 0$. Consider a vector $u = \begin{bmatrix} 1 & 0 \end{bmatrix}$ and another vector $v = \begin{bmatrix} \sqrt{1 - \epsilon^2} & \epsilon \end{bmatrix}$ that is nearly parallel to $u$, i.e., $\epsilon \ll 1$ (implies $\vert \theta \vert$ is $\mathcal{O}(\epsilon)$). We aim to show that the angle between these two vectors under the generalized inner product $\langle \cdot, \cdot \rangle_C$ is upper bounded by a constant multiple of $\epsilon$ given a small enough $\epsilon$.

Consider
\begin{align}
    C^{\frac{1}{2}} u &= \begin{bmatrix} a \\ c \end{bmatrix} ,& C^{\frac{1}{2}} v &= \begin{bmatrix} a \sqrt{1 - \epsilon}^2 + c \epsilon \\ c \sqrt{1 - \epsilon}^2 + b \epsilon \end{bmatrix},
    \\ ||C^{\frac{1}{2}} u|| &= \sqrt{a^2 + c^2} & ||C^{\frac{1}{2}} v|| &= \sqrt{a^2 (1 - \epsilon^2) + c^2 \epsilon^2 + b^2 \epsilon^2 + c^2 (1 - \epsilon^2)}
    \\ &= a \sqrt{1 + (c/a)^2} ,& &= \sqrt{a^2 + c^2 + \epsilon^2(b^2 - a^2)}
    \\ && &= a \sqrt{1 + (c/a)^2 + \epsilon^2((b/a)^2 - 1)},
\end{align}

and
\begin{align}
    \langle C^{\frac{1}{2}} u, C^{\frac{1}{2}} v \rangle &= a^2 \sqrt{1 - \epsilon^2} + ac\epsilon + c^2 \sqrt{1 - \epsilon^2} bc\epsilon
    \\ &= (a^2 + c^2) \sqrt{1 - \epsilon^2} + (a + b) c\epsilon
    \\ &= a^2 (1 + (c/a)^2) \sqrt{1 - \epsilon^2} + a^2 (1 + (b/a)) (c/a) \epsilon
    \\ &\overbrace{\ge}^{\vert \epsilon \vert < 1} a (1 + (c/a)^2) (1 - \epsilon^2) + a^2 (1 + (b/a)) (c/a) \epsilon.
\end{align}

Then
\begin{align}
    \frac{\langle C^{\frac{1}{2}} u, C^{\frac{1}{2}} v \rangle}{||C^{\frac{1}{2}} u|| ||C^{\frac{1}{2}} v||} &\ge \frac{(1 + (c/a)^2) (1 - \epsilon^2) + (1 + (b/a)) (c/a) \epsilon}{(1 + (c/a)^2) \sqrt{1 + \epsilon^2\frac{(b/a)^2 - 1}{(c/a)^2 + 1}}}
    \\ &= \frac{1}{\sqrt{1 + \epsilon^2\frac{(b/a)^2 - 1}{(c/a)^2 + 1}}} + \frac{(1 + (b/a)) (c/a) \epsilon - (1 + (c/a)^2) \epsilon^2}{(1 + (c/a)^2) \sqrt{1 + \epsilon^2\frac{(b/a)^2 - 1}{(c/a)^2 + 1}}}
    \\ &\overbrace{\ge}^{b/a \le \kappa} \frac{1}{\sqrt{1 + \epsilon^2 \kappa^2}} + \frac{(1 + (b/a)) (c/a) \epsilon - (1 + (c/a)^2) \epsilon^2}{(1 + (c/a)^2) \sqrt{1 + \epsilon^2\frac{(b/a)^2 - 1}{(c/a)^2 + 1}}}
    \\ &= 1 - (1 - \frac{1}{\sqrt{1 + \epsilon^2 \kappa^2)}}) + \frac{(1 + (b/a)) (c/a) \epsilon - (1 + (c/a)^2) \epsilon^2}{(1 + (c/a)^2) \sqrt{1 + \epsilon^2\frac{(b/a)^2 - 1}{(c/a)^2 + 1}}}
    \\ &\overbrace{\ge}^{c \ge -\sqrt{ab}}  1 - (1 - \frac{1}{\sqrt{1 + \epsilon^2 \kappa^2)}}) - \frac{(1 + (b/a)) \sqrt{(b/a)} \vert \epsilon \vert - (1 + (c/a)^2) \epsilon^2}{(1 + (c/a)^2) \sqrt{1 + \epsilon^2\frac{(b/a)^2 - 1}{(c/a)^2 + 1}}}
    \\ &= 1 - (1 - \frac{1}{\sqrt{1 + \epsilon^2 \kappa^2)}}) - \frac{(1 + (b/a)) \sqrt{(b/a)} \vert \epsilon \vert}{(1 + (c/a)^2) \sqrt{1 + \epsilon^2\frac{(b/a)^2 - 1}{(c/a)^2 + 1}}} - \frac{\epsilon^2}{\sqrt{1 + \epsilon^2\frac{(b/a)^2 - 1}{(c/a)^2 + 1}}}
    \\ &\overbrace{\ge}^{b/a \le \kappa} 1 - (1 - \frac{1}{\sqrt{1 + \epsilon^2 \kappa^2}}) - \frac{(1 + \kappa) \sqrt{\kappa} \vert \epsilon \vert}{\sqrt{1 - \epsilon^2}} - \frac{\epsilon^2}{\sqrt{1 - \epsilon^2}}
    \\ &\overbrace{\ge}^{1 - \frac{1}{\sqrt{1 + z}} \le \frac{z}{2} , \epsilon^2 < \frac{1}{\kappa^2}} 1 - \frac{\epsilon^2}{2} \kappa^2 - \frac{(1 + \kappa) \sqrt{\kappa} \vert \epsilon \vert}{\sqrt{1 - \epsilon^2}} - \frac{\epsilon^2}{\sqrt{1 - \epsilon^2}}
    \\ &\overbrace{\ge}^{\epsilon < \frac{1}{2}} 1 - 2 (1 + \kappa) \sqrt{\kappa} \vert \epsilon \vert - \frac{\epsilon^2}{2} (\kappa^2 - 1) - 2\epsilon^2.
\end{align}

Note that $\cos(\theta) \le 1 - \frac{1}{8} \theta^2$ for $\vert \theta \vert \le \pi$. Then
\begin{align}
    \vert \theta \vert &\le 2 \sqrt{2} \sqrt{1 - \cos(\theta)} \le 2 \sqrt{2} \sqrt{2 (1 + \kappa) \sqrt{\kappa} \vert \epsilon \vert + (\frac{\kappa^2}{2} + 2) \epsilon^2}
    \\ &\le 4 \sqrt{(1 + \kappa) \sqrt{\kappa} \vert \epsilon \vert + (\frac{\kappa^2}{4} + 1) \epsilon^2}.
\end{align}

Therefore, for $\epsilon < \min(\frac{1}{\kappa}, \frac{1}{2})$, $\arccos(\frac{\vert \langle C^{\frac{1}{2}} u, C^{\frac{1}{2}} v \rangle \vert}{||C^{\frac{1}{2}} u|| ||C^{\frac{1}{2}} v||})$ is upper bounded by $\mathcal{O}(\epsilon^{\frac{1}{2}} \kappa^{\frac{3}{4}})$.

\end{proof}
\section{Asymptotic Convergence}\label{app:conv}

We carryout a proof of convergence of \Algref{alg:det_geg} in the deterministic setting and a partial proof of \Algref{alg:gen_eg} with further discussion.

\subsection{Asymptotic Convergence of Determinstic Update}\label{app:conv_det}

We now give the convergence proof of \Algref{alg:det_geg} using the theoretical results established above.

\begin{reptheorem}{thm:global_conv}[Deterministic / Full-batch Global Convergence] %
Given a symmetric matrix $A$ and symmetric positive definite matrix $B$ where the top-$k$ eigengaps of $B^{-1}A$ are positive along with a square-summable, not summable step size sequence $\eta_t$ (e.g., $1/t$), \Algref{alg:det_geg} converges to the top-$k$ eigenvectors asymptotically ($\lim_{T \rightarrow \infty}$) with probability $1$.
\end{reptheorem}
\begin{proof}
Assume none of the $\hat{v}_i$ are initialized to an angle exactly at the minimum of their utility. This is a set of vectors with Lebesgue measure $0$, therefore, the assumption holds w.p.1.

Denote the ``update field'' $H(\hat{V})$ to match the work of~\citep{shah2019stochastic}. $H(\hat{V})$ is simply the concatenation of all players’ Riemannian update rules, i.e., all players updating in parallel using their Riemannian updates:
\begin{align}
H(\hat{V}) &= [\Delta_1, \ldots, \Delta_k] : \mathbb{R}^{kd} \rightarrow \mathbb{R}^{kd}
\end{align}
where $\Delta_i$ is defined in~\eqref{eq:update} and $\hat{V}$ represents the set of all $\hat{v}_i$.

A Riemannian gradient ascent step (with retractions) is then given by the following update step:
\begin{align}
    \hat{V}(t+1) &\leftarrow \hat{V}(t) + \eta_t H(\hat{V}(t))
    \\ \hat{v}_i(t+1) &\leftarrow \hat{v}_i(t+1) / ||\hat{v}_i(t+1)|| \quad \forall i.
\end{align}

By Lemma~\ref{lemma:right_fp}, $v_1$ is the unique fixed point of $\hat{v}_1$'s update. And by Theorem~\ref{thm:error_prop}, convergence of $v_1$ to within $\mathcal{O}(\epsilon)$ of its fixed point contributes to a mis-specification of children $\hat{v}_{j>1}$'s fixed point by $\mathcal{O}(\sqrt{\epsilon})$. Critically, this mis-specification is shrinking in $\epsilon$ so that as $\hat{v}_1$ nears its fixed point, so may its children. This chain of reasoning applies for all $\hat{v}_i$.

The result is then obtained by applying Theorem 7 of~\citet{shah2019stochastic} with the following information: \textbf{A0}) the unit-sphere is a compact manifold with an injectivity radius of $\pi$ which implies the injectivity radius of the manifold of the game (the product space of $k$ unit-spheres) is also finite, \textbf{A1}) the update field is smooth (analytic) by Lemma~\ref{lemma:smooth}, \textbf{A2}) we assume a square-summable, not summable step size, \textbf{A3}) we assume the full-batch (noiseless) setting so the update ``noise" clearly constitutes a bounded martingale difference sequence, and \textbf{A4}) the iterates remain bounded because they are constrained to the unit-sphere.

Formally, for any $T > 0$,
\begin{align}
    \lim_{s \rightarrow \infty} \sup_{t \in [s, s+T]} d(\hat{V}(t), \hat{V}^s(t)) \rightarrow a.s.,
\end{align}
where $d(\cdot, \cdot)$ is the Riemannian distance on the (product space of) sphere and $\hat{V}^s(t)$ denotes the continuous time trajectory of $H(\hat{V})$ starting from $\hat{V}(s)$.
\end{proof}

\subsection{Asymptotic Convergence of Stochastic Update}\label{app:conv_sto}

There are two primary issues with extending the asymptotic convergence guarantee in Theorem~\ref{thm:global_conv} to \Algref{alg:gen_eg}. The first is that the joint parameter space includes $\hat{v}_i \in \mathcal{S}^{d-1}$ and $[B\hat{v}]_i \in \mathbb{R}^d$. The unit-sphere, $\mathcal{S}^{d-1}$, is a compact Riemmanian manifold. While $\mathbb{R}^d$ is a Riemannian manifold, it is not compact. This violates assumptions \textbf{A0} and \textbf{A4} above. The second issue is that the vector field $H(\hat{V}; [B\hat{V}])$ is not smooth due to the clipped denominator terms of $y_j$ (see line~\ref{eq:clipping} of \Algref{alg:gen_eg}). We can easily fix this issue with a change of variables, defining $\langle \hat{v}_i, B \hat{v}_i \rangle$ in log-space. This results in \Algref{alg:gen_eg_mod}.

Specifically, define $[\langle v, B v \rangle]_j = e^{\log(\rho) + (\log(\nu) - \log(\rho))\texttt{sigmoid}(z_j)}$ where $z_j$ is a newly introduced auxiliary variable. The relevant gradient with respect $z_j$ is $\nabla_{z_j} = -(\langle v_j, B v_j \rangle - [\langle v, B v \rangle]_j) \cdot [\langle v, B v \rangle]_j \cdot (\log \nu - \log \rho) \cdot \texttt{sigmoid}(z_j) \cdot (1 - \texttt{sigmoid}(z_j))$.

Regarding the still unresolved first issue, we could constraint $[B\hat{v}]_i$ to a ball with radius $\lambda_{\max}(B)$ centered at the origin, which is a convex set. Note that while a ball in $\mathbb{R}^d$ is compact, it is not a Riemannian manifold anymore. A few works have developed theory for the setting that mixes convex and Riemannian optimization~\citep{liu2017accelerated,goyal2019sampling}. Intuitively, we do not expect issues arising in our setting from the mixture of feasible sets, however, progress towards theoretic results takes time. We conjecture that \Algref{alg:gen_eg_mod} is provably asymptotically convergent, although \Algref{alg:gen_eg} defined with clipping is a bit more practical.

\begin{algorithm}[!ht]
\begin{algorithmic}[1]
    \STATE Given: paired data streams $X_t \in \mathbb{R}^{b \times d_x}$ and $Y_t \in \mathbb{R}^{b \times d_y}$, number of parallel machines $M$ per player (minibatch size per machine $b'=\frac{b}{M}$), step size sequence $\eta_t$, scalar $\rho$ lower bounding $\lambda_{min}(B)$, scalar, $\nu$ upper bounding $\lambda_{max}(B)$, and number of iterations $T$.
    \STATE $\hat{v}_i \sim \mathcal{S}^{d-1}$, i.e., $\hat{v}_i \sim \mathcal{N}(\mathbf{0}_d, \mathbf{I}_d); \hat{v}_i \leftarrow \hat{v}_i / ||\hat{v}_i|$ for all $i$
    \STATE $[B\hat{v}]_i \in \mathbb{R}^d \leftarrow \hat{v}_i$ for all $i$
    \STATE $z_i^0 \in \mathbb{R} \leftarrow 0$ for all $i$
    \FOR{$t = 1: T$}
        \PARFOR{$i = 1: k$}
            \PARFOR{$m = 1: M$}
            \STATE Construct $A_{tm}$ and $B_{tm}$
            \STATE $[\langle v, B v \rangle]_j = e^{\log(\rho) + (\log(\nu) - \log(\rho)) \sigma(z_j)}$
            \STATE $\hat{y}_j = \frac{\hat{v}_j}{\sqrt{[\langle v, B v \rangle]_j}}$
            \STATE $[B\hat{y}]_j = \frac{[B\hat{v}]_j}{\sqrt{[\langle v, B v \rangle]_j}}$
            \STATE $\texttt{rewards} \leftarrow (\hat{v}_i^\top B_{tm} \hat{v}_i) A_{tm} \hat{v}_i - (\hat{v}_i^\top A_{tm} \hat{v}_i) B_{tm} \hat{v}_i$
            \STATE $\texttt{penalties} \leftarrow \sum_{j < i} (\hat{v}_i^\top A \hat{y}_j) \big[ \langle \hat{v}_i, B_{tm} \hat{v}_i \rangle [B \hat{y}]_j - \langle \hat{v}_i, [B \hat{y}]_j \rangle B_{tm} \hat{v}_i \big]$
            \STATE $\tilde{\nabla}_{im} \leftarrow \texttt{rewards} - \texttt{penalties}$
            \STATE $\nabla^{Bv}_{im} = (B_{tm}\hat{v}_i - [B\hat{v}]_i)$
            \STATE $\nabla^{z}_{im} = (\langle v_i, [Bv]_i \rangle - [\langle v, B v \rangle]_i) \cdot [\langle v, B v \rangle]_i \cdot (\log \nu - \log \rho) \cdot \sigma(z_i) \cdot (1 - \sigma(z_i))$
            \ENDPARFOR
        \STATE $\tilde{\nabla}_{i} \leftarrow \frac{1}{M} \sum_{m} [ \tilde{\nabla}_{im} ]$
        \STATE $\hat{v}_i' \leftarrow \hat{v}_i + \eta_t \tilde{\nabla}_{i}$
        \STATE $\hat{v}_i \leftarrow \frac{\hat{v}_i'}{|| \hat{v}_i' ||}$
        \STATE $\nabla^{Bv}_{i} \leftarrow \frac{1}{M} \sum_m [\nabla^{Bv}_{im}]$
        \STATE $[B\hat{v}]_i \leftarrow [B\hat{v}]_i + \gamma_t \nabla^{Bv}_{i}$
        \STATE $\nabla^{z}_{i} \leftarrow \frac{1}{M} \sum_m [\nabla^{z}_{im}]$
        \STATE $z_i \leftarrow z_i + \gamma_t \nabla^{z}_{i}$
        \ENDPARFOR
    \ENDFOR
    \STATE return all $\hat{v}_i$
\end{algorithmic}
\caption{Smooth Stochastic \sgeg{}}
\label{alg:gen_eg_mod}
\end{algorithm}

\section{Alternative Parallelized Implementation}\label{app:parallel}

As mentioned in Section~\ref{sec:experiments}, we plan to open source our implementation, specifically the implementation used to conduct the neural CCA experiments described in Section~\ref{sec:neural_cca}. The specific parallelization we used was different than that implied by \Algref{alg:gen_eg}. Instead, we parallelized the estimation of the matrix-vector products $A\hat{v}_i$ and $B\hat{v}_i$ for all $i$ and then aggregated this information across machines. Figure~\ref{fig:parallel_diagram} provides a diagram illustrating how data and algorithmic operations are distributed.

\begin{figure}[ht!]
    \centering
    \includegraphics[width=0.8\textwidth]{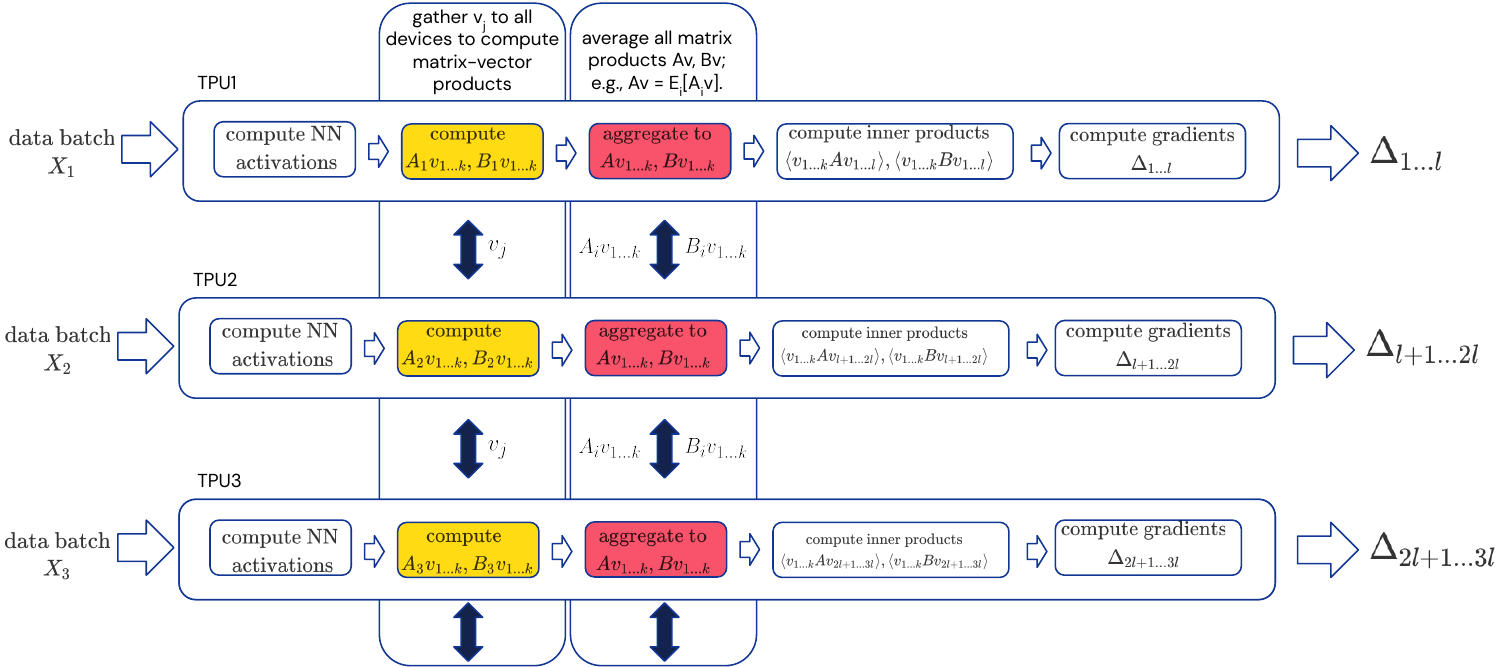}
    \caption{Parallelization of \sgeg{} implementation for neural CCA experiments in Section~\ref{sec:neural_cca}.}
    \label{fig:parallel_diagram}
\end{figure}
\section{Additional Experiments}\label{app:more_experiments}

\subsection{Regularization Effects of Stochastic Approximation}\label{app:more_experiments:sgd_reg}
In pursuit of understanding the regularization benefits of our stochastic approximation approach with a fixed step size, we explore various ways of regularizing the matrices $A$ and $B$ prior to calling \texttt{scipy.linalg.eigh(A, B)} to see if we can achieve a similar solution. Figure~\ref{fig:more_ica} explores various parameter settings, but none adequately recover the original signals.

\begin{figure}[ht!]
    \centering
    \includegraphics[width=0.9\textwidth]{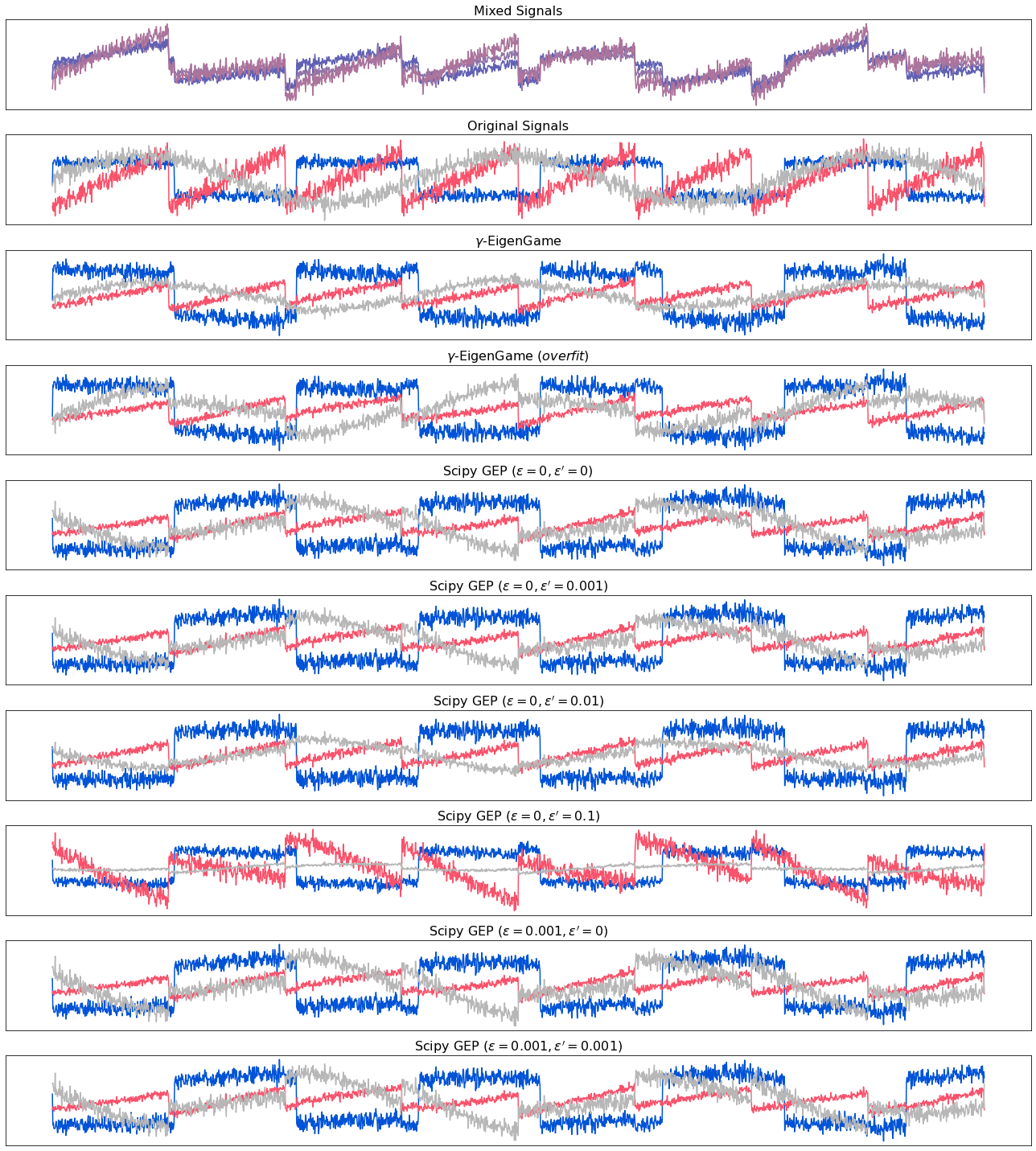}
    \caption{Figure~\ref{fig:ica} repeated with additional regularized versions of \texttt{scipy.linalg.eigh}.}
    \label{fig:more_ica}
\end{figure}

We parameterize our regularizations as follows. Let $C = \mathbb{E}_t[x_t x_t^\top]$. Then
\begin{align}
    A& = \mathbb{E}_t[\langle x_t, x_t \rangle x_t x_t^\top] - tr(C + \epsilon) (C + \epsilon) - 2 (C + \epsilon)^2 & B &= C + \epsilon'.
\end{align}

\subsection{\textcolor{highlight}{Parallel vs Sequential Learning}}\label{app:more_experiments:seq_vs_par}

\begin{figure}[ht!]
    \centering
    \includegraphics[width=0.325\textwidth]{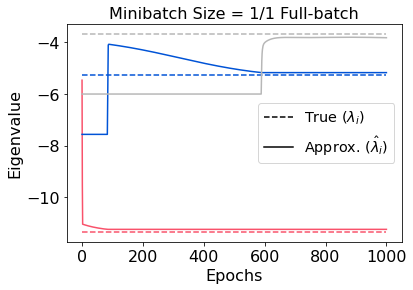}
    \includegraphics[width=0.325\textwidth]{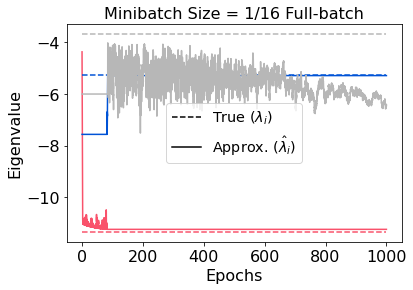}
    \includegraphics[width=0.325\textwidth]{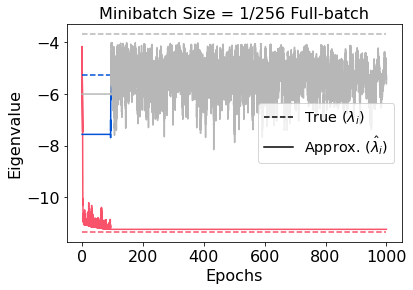}
    \caption{In contrast to the parallel learning approach we propose in~\Algref{alg:gen_eg}, here, we explore a sequential, deflation-inspired approach where each eigenvector completes learning only once its Rayleigh quotient (eigenvalue) is within $0.1$ of the true value (we chose this value of $0.1$ because full batch EigenGame (Figure~\ref{fig:ica} (left)) reaches at least $0.1$ accuracy for all $3$ eigenvalues by the end of training). Once this eigenvector has completed learning, the next eigenvector then begins learning. This sequential approach fails to learn the third eigenvalue to within the $0.1$ threshold in both minibatch settings.}
    \label{fig:par_ica}
\end{figure}

The parallel approach provides a few advantages over the sequential approach. 1) Intuitively, the parallel approach is similar to the sequential approach but with ``warm starting''. Child eigenvectors are allowed to learn while their parents are learning, which puts them in a good position to reach their correct directions once their parents have learned. In contrast, a sequential approach would randomly initialize the child once the parents have converged, leaving the child to traverse a longer geodesic to reach its true destination. 2) How do you know when the parents are done learning? You would need to measure convergence of the eigenvectors and this is difficult in the stochastic setting. You could use a running mean of the Riemannian gradient norm or of the difference in successive Rayleigh quotients, but this is approximate. This is an interesting challenge for future research.

In Figure~\ref{fig:par_ica}, we assume we know the true eigenvalues and use this information to decide when to deflate. In this experiment, the first two eigenvalues are approximated well enough, but this level of accuracy is not high enough to allow learning the third eigenvector accurately. This supports our argument where knowing when to stop learning and deflate is a difficult problem because the accuracy of parents affects the learning of children in ways that depend on the spectrum (which is unknown in any practical setting).
\section{Hyperparameters and Experiment Details}\label{app:hyps}

\subsection{ICA}

Details of the unmixing experiment we run can be found on the scikit-learn website: \href{https://scikit-learn.org/stable/auto_examples/decomposition/plot_ica_blind_source_separation.html}{../auto\_examples/decomposition/plot\_ica\_blind\_source\_separation.html}~\citep{scikit-learn}. We solve for top-$3$ \gep{} formulation of ICA using $n=2000$ samples taken from the time series data. Minimal hyperparameter tuning was performed. Learning rates were searched over orders of magnitude (e.g., $0.01, 0.1, 1.0, \ldots$).

\subsubsection{Comparison}

The parameters used for \Algref{alg:gen_eg} (\sgeg{}) in Figure~\ref{fig:ica} are listed in Table~\ref{tab:ica_comparison_eg}. Those for overfitting \sgeg{} to the data are listed in Table~\ref{tab:ica_comparison_eg_overfit}.

\begin{table}[ht!]
    \centering
    \begin{tabular}{|c|c|}
        \multicolumn{2}{c}{Algorithm Parameters} \\ \hline
        batch size $b$ & $\frac{n}{4}$ \\ \hline
        $M$ & $1$ \\ \hline
        \# of iterations (T) & $10^{3} \cdot \frac{n}{b}$ \\ \hline
        $\eta_t$ & $10^{-2} \cdot \frac{b}{n}$ \\ \hline
        $\beta_t$ & $1 \cdot \frac{b}{n}$ \\ \hline
    \end{tabular}
    \vspace{0.2cm}
    \caption{\Algref{alg:gen_eg} hyperparameters for \sgeg{} in Figure~\ref{fig:ica}.}
    \label{tab:ica_comparison_eg}
\end{table}

\begin{table}[ht!]
    \centering
    \begin{tabular}{|c|c|}
        \multicolumn{2}{c}{Algorithm Parameters} \\ \hline
        batch size $b$ & $n$ \\ \hline
        $M$ & $1$ \\ \hline
        \# of iterations (T) & $10^{5}$ \\ \hline
        $\eta_t$ & $10^{-3}$ \\ \hline
        $\beta_t$ & $1$ \\ \hline
    \end{tabular}
    \vspace{0.2cm}
    \caption{\Algref{alg:gen_eg} hyperparameters for \sgeg{} (\emph{overfit}) in Figure~\ref{fig:ica}.}
    \label{tab:ica_comparison_eg_overfit}
\end{table}

\subsubsection{Unbiased}

Each of the plots in Figure~\ref{fig:unbiased_update} uses a different minibatch size $b$; the hyperparameters used for \Algref{alg:gen_eg} are listed in Table~\ref{tab:ica_unbiased_eg} as a function of $b$.

\begin{table}[ht!]
    \centering
    \begin{tabular}{|c|c|}
        \multicolumn{2}{c}{Algorithm Parameters} \\ \hline
        $M$ & $1$ \\ \hline
        \# of iterations (T) & $10^{3} \cdot \frac{n}{b}$ \\ \hline
        $\eta_t$ & $10^{-2} \cdot \frac{b}{n}$ \\ \hline
        $\beta_t$ & $1 \cdot \frac{b}{n}$ \\ \hline
    \end{tabular}
    \vspace{0.2cm}
    \caption{\Algref{alg:gen_eg} hyperparameters for Figure~\ref{fig:unbiased_update}.}
    \label{tab:ica_unbiased_eg}
\end{table}

\subsection{CCA}

Details of both CCA experiments can be found below.

\subsubsection{Comparison}

Hyperparameters for the algorithm proposed in~\citep{meng2021online} are the same as in their paper. Their code is available on github at \href{https://github.com/zihangm/riemannian-streaming-cca}{.../zihangm/riemannian-streaming-cca}. We ran experiments $10$ times to produce the means and standard deviation shading in Figure~\ref{fig:ica}. We run PCA first on the data to remove the subspaces in $X$ and $Y$ with zero variance. We then solve the top-$4$ \gep{} formulation of CCA. Hyperparameters were tuned manually, searching over orders of magnitude (e.g., $0.01, 0.1, 1.0$ and in some cases $0.05, 0.5, 5.0$). We set $\rho=10^{-10}$.

\begin{table}[ht!]
    \centering
    \begin{tabular}{|c|c|}
        \multicolumn{2}{c}{Shared parameters} \\ \hline
        $b$ & $100$ \\ \hline
        $M$ & $1$ \\ \hline
        $T$ & $\frac{n}{b}$ \\ \hline
        \multicolumn{2}{c}{MNIST} \\ \hline
        $\eta_t$ & $0.1$ \\ \hline
        $\beta_t$ & $1.0$ \\ \hline
        \multicolumn{2}{c}{Mediamill} \\ \hline
        $\eta_t$ & $50$ \\ \hline
        $\beta_t$ & $5$ \\ \hline
        \multicolumn{2}{c}{CIFAR-10} \\ \hline
        $\eta_t$ & $0.1$ \\ \hline
        $\beta_t$ & $0.1$ \\ \hline
    \end{tabular}
    \vspace{0.2cm}
    \caption{\Algref{alg:gen_eg} hyperparameters for Figure~\ref{fig:ica}.}
    \label{tab:cca_comparison}
\end{table}

\subsubsection{Neural Network Analysis}

We trained two CNNs for the $d>10^{3}$ and $d>10^{5}$ CCA (top-$1024$ \gep{}) experiments in Figure~\ref{fig:cifar10_analysis}. Details of both architectures are listed in Table~\ref{tab:cnn_arch}.

\begin{table}[ht!]
    \centering
    \begin{tabular}{|c|c|}
        \multicolumn{2}{c}{$d = 2048 > 10^{3}$ - Figure~\ref{fig:cifar10_analysis} (left)} \\ \hline
        Activations Harvested & Last convolutional layer and dense layer \\ \hline
        Conv Layer Output Channels & [64, 32] \\ \hline
        Conv Strides & [2, 1] \\ \hline
        Dense Layer Sizes & [512] \\ \hline
        Total Activations & $d_x = d_y = 1024$ \\ \hline
        \multicolumn{2}{c}{$d = 116736 >10^{5}$ - Figure~\ref{fig:cifar10_analysis} (right)} \\ \hline
        Activations Harvested & All convolutional layers and dense layer \\ \hline
        Conv Layer Output Channels & [128, 256, 512] \\ \hline
        Conv Strides & [1, 1, 1] \\ \hline
        Dense Layer Sizes & [1024] \\ \hline
        Total Activations & $d_x = d_y = 58368$ \\ \hline
    \end{tabular}
    \vspace{0.2cm}
    \caption{CNN architecture parameters for CIFAR-10 Neural CCA experiment.}
    \label{tab:cnn_arch}
\end{table}

We used the hyperparameters listed in Table~\ref{tab:neural_cca} for running \sgeg{}.

\begin{table}[ht!]
    \centering
    \begin{tabular}{|c|c|}
        \multicolumn{2}{c}{Algorithm Parameters} \\ \hline
        $b$ & $2048$ ($256$ per device)  \\ \hline
        $M$ & $8$ ($2\times2$ TPU = $4$ chips, $2$ devices/chip) \\ \hline
        $T$ & $10^{7}$ \\ \hline
        $\eta_t$ & $t_c=10^{5}$, $\eta_0=10^{-4}$, $\eta_T=10^{-6}$ \\ \hline
        $\beta_t$ & $10^{-3}$ \\ \hline
        $\epsilon$ & $10^{-4}$ \\ \hline
        $\rho$ & $10^{-6}$ \\ \hline
    \end{tabular}
    \vspace{0.2cm}
    \caption{\Algref{alg:gen_eg} (with parallelism modifications from Section~\ref{app:parallel}) hyperparameters for Figure~\ref{fig:cifar10_analysis}.}
    \label{tab:neural_cca}
\end{table}

\texttt{Adam}($b1=0.9, b2=0.999, \epsilon=10^{-8}$) was used for learning $\hat{v}_i$. We pair Adam with a learning rate schedule that consists of separate warmup and harmonic decay phases. We have a warmup period for the eigenvectors while the auxiliary variables and mean estimates are learned. During this period, the learning rate increases linearly until it reaches the base learning rate after the period ends (iteration $t_c$). This is followed by a decaying learning rate ($\propto \frac{1}{t + \Delta t}$) which reaches the final learning rate at iteration $T$.

For the purpose of generating the plots, we estimated Rayleigh quotients with a larger batch size than that used to estimate the eigenvectors themselves; specifically, we used $2048$ for evaluation vs $256$ for training.

Both experiments were $100\%$ input bound, meaning the bottleneck in speed was computing neural network activations and passing them in minibatches to our algorithm. Therefore, our current runtimes are not indicative of the complexity of our proposed update rule. Alternatively, one could precompute all activations and save them to disk, however, this is memory intensive and we chose not to do this. For completeness, the $d>10^{3}$ dimensional experiment ran at $4.7$ ms per step on average, and the $10^{5}$ experiment at $30.2$ ms per step.

We do not have a breakdown of the runtime that separates CIFAR-10 data loading from neural network evaluation (computing activations) from EigenGame update from other processes. However, we can share that overall, the $d>10^{3}$ dimensional experiment ran at $4.7$ ms per step on average, and the $10^{5}$ experiment at $30.2$ ms per step.
\section{Runtime Complexity}\label{app:complexity}

\Algref{alg:gen_eg} states under line 1 "Given" that it expects "number of parallel machines $M$ per player". There are $k$ players, established in Section~\ref{sec:game}. This means \Algref{alg:gen_eg} expects $p = kM$ processors. This can also be inferred by noticing the two parallel for-loops (parfor) on lines 5 and 6 of \Algref{alg:gen_eg}. The paragraph on "Computational Complexity and Parallelization" then goes on to consider the case where "each player (model) parallelizes over $M=b$ machines" where $b$ is the batch size. Again, referring back to \Algref{alg:gen_eg}, it is written "minibatch size per machine $b'=\frac{b}{M}$", therefore, $b'=1$. We can now consider the computational cost of each line of \Algref{alg:gen_eg}, which is dominated by the matrix-vector products on lines 8-11. Recall that $A_{tm}$ and $B_{tm}$ are both formed as outer products, e.g., $B_{tm} = X_{tm}^\top X_{tm}$ in ICA with $X_{tm} \in \mathbb{R}^{b' \times d}$. Given that we are considering the case where $b'=1$, $B_{tm}$ can be rewritten as $x_{tm} x_{tm}^\top$ where $x_{tm} \in \mathbb{R}^d$ to make it clear that the $1 \times d$ matrices are vectors. Therefore, all matrix-vector products (and inner products) on lines 8-11 cost $\mathcal{O}(d)$. There is $1$ on line 8, $1$ on line 9, $4$ on line 10 (rewards), and $4k$ on line 11 (penalties) making for a total cost of $\mathcal{O}(dk)$.

\end{appendices}

\end{document}